\documentclass{article}
\usepackage[margin=1in]{geometry}
\usepackage{authblk}

\usepackage{hyperref}
\usepackage{url}

\usepackage{microtype}
\usepackage{graphicx}
\usepackage{float}
\usepackage{booktabs} 
\usepackage{bbding}

\usepackage{hyperref}
\usepackage{array}
\newcolumntype{P}[1]{>{\centering\arraybackslash}p{#1}}

\usepackage{amsmath}
\usepackage{amssymb}
\usepackage{mathtools}
\usepackage{amsthm}
\usepackage[authoryear]{natbib}
\usepackage{enumerate}
\usepackage{bbm}

\usepackage{caption}

\usepackage{booktabs} 
\usepackage{tikz} 

\usetikzlibrary{shapes.geometric, arrows}

\tikzstyle{startstop} = [rectangle, rounded corners, 
minimum width=1cm, 
minimum height=0.4cm,
text centered, 
draw=black ]

\tikzstyle{io} = [trapezium, 
trapezium stretches=true, 
trapezium left angle=70, 
trapezium right angle=110, 
minimum width=3.0cm, 
minimum height=1cm, text centered, 
draw=black]

\tikzstyle{process} = [rectangle, 
minimum width=3.0cm, 
minimum height=1cm, 
text centered, 
text width=3.0cm, 
draw=black, 
fill]

\tikzstyle{decision} = [diamond, 
minimum width=3.0cm, 
minimum height=1cm, 
text centered, 
draw=black, 
fill=green!30]
\tikzstyle{arrow} = [thick,->,>=stealth]

\usepackage[american]{babel}
\usepackage{amsmath}
\usepackage{amsthm}
\usepackage{amssymb} 
\usepackage[T1]{fontenc} 
\usepackage[capitalise]{cleveref}
\usepackage{ifthen}

\usepackage[ ]{algorithmic}
\usepackage[vlined,ruled, ]{algorithm2e}

\usetikzlibrary{shapes,decorations,arrows,calc,arrows.meta,fit,positioning}
\tikzset{
    -Latex,auto,node distance =1 cm and 1 cm,semithick,
    state/.style ={ellipse, draw, minimum width = 0.7 cm},
    point/.style = {circle, draw, inner sep=0.04cm,fill,node contents={}},
    bidirected/.style={Latex-Latex,dashed},
    el/.style = {inner sep=2pt, align=left, sloped}
}

\makeatletter
\newcommand*{\centernot}{%
  \mathpalette\@centernot
}
\def\@centernot#1#2{%
  \mathrel{%
    \rlap{%
      \settowidth\dimen@{$\m@th#1{#2}$}%
      \kern.5\dimen@
      \settowidth\dimen@{$\m@th#1=$}%
      \kern-.5\dimen@
      $\m@th#1\not$%
    }%
    {#2}%
  }%
}
\makeatother

\newtheorem{theorem}{Theorem}
\newtheorem{corollary}{Corollary}

\newtheorem{lemma}{Lemma}

\newtheorem{definition}{Definition}

\usepackage{tikz}

\usepackage[textsize=tiny]{todonotes}

\usepackage[utf8]{inputenc} 
\usepackage[T1]{fontenc}    
\usepackage{hyperref}       
\usepackage{url}            
\usepackage{booktabs}       
\usepackage{amsfonts}       
\usepackage{nicefrac}       
\usepackage{microtype}      
\usepackage{lipsum}
\usepackage{fancyhdr}       
\usepackage{graphicx}       
\graphicspath{{media/}}     

\usepackage{caption}
\usepackage{subcaption}
\usepackage{comment}

\usepackage{amsmath} 
\usepackage{xspace}

\usepackage{amsmath}
\usepackage{amssymb}
\usepackage{mathtools}
\usepackage{amsthm}
\theoremstyle{plain}

\usepackage{lipsum,lineno}
\usepackage{tabularray}

\usepackage[utf8]{inputenc} 
\usepackage[T1]{fontenc}    
\usepackage{hyperref}       
\usepackage{url}            
\usepackage{booktabs}       
\usepackage{amsfonts}       
\usepackage{nicefrac}       
\usepackage{microtype}      

\usepackage{amsmath}
\usepackage{amssymb}
\usepackage{mathtools}
\usepackage{amsthm}

\usepackage{xargs}                      
\usepackage{xcolor}

\usepackage{tikz}
\usetikzlibrary{positioning, calc, shapes.geometric, shapes, shapes.multipart, arrows.meta, arrows, decorations.markings, external, trees}

\usepackage{scalefnt}
\usetikzlibrary{shapes.misc}
\usetikzlibrary{positioning, calc, shapes.geometric, shapes, shapes.multipart, arrows.meta, arrows, decorations.markings, external, trees, fit}
\tikzset{
	-Latex,auto,node distance =1 cm and 1 cm,semithick,
	state/.style ={ellipse, draw, minimum width = 0.7 cm},
	point/.style = {circle, draw, inner sep=0.04cm,fill,node contents={}},
	nnv/.style={
		rectangle, draw,thick,minimum width=0.7cm,minimum height=1.5cm
	},
        nnh/.style={
            rectangle, draw,thick,minimum width=1.5cm,minimum height=1.0cm
          },
	outer/.style={draw=gray,dashed,thick, inner sep=3pt
	},
	XOR/.style={draw,circle,append after command={
			[shorten >=\pgflinewidth, shorten <=\pgflinewidth,]
			(\tikzlastnode.north) edge (\tikzlastnode.south)
			(\tikzlastnode.east) edge (\tikzlastnode.west)
		}
	},
	bidirected/.style={Latex-Latex,dashed},
	el/.style = {inner sep=2pt, align=left, sloped},
	cross/.style={cross out, draw=black, minimum size=2*(#1-\pgflinewidth), inner sep=0pt, outer sep=0pt},
	cross/.default={1pt}
}

\usepackage{natbib}
\usepackage{microtype}

\title{
Information-Directed Sampling for Causal Bandits
}

\author{%
  Muhammad Qasim Elahi$^1$, Murat Kocaoglu$^2$,  Mahsa Ghasemi$^1$ \\
       School of Electrical and Computer Engineering, Purdue University$^1$\\
       School of Computer Science, Johns Hopkins University$^2$\\

  \texttt{elahi0@purdue.edu, mkocaoglu@jhu.edu, mahsa@purdue.edu } \\
}

\usepackage{fancyhdr}
\begin{document}
\maketitle

\begin{abstract}
Causal bandits exploit structural relationships among variables to share
information across interventions and accelerate the identification of
high-reward decisions. In many applications, however, some variables
cannot be directly manipulated, even though they influence the reward
and provide useful information about the underlying causal system. We
study contextual causal bandits with non-manipulable variables, where
context variables are observed before action selection and additional
variables are observed after each intervention. Assuming a known causal
graph without latent confounding, we adopt a Bayesian formulation in
which the conditional probability tables of the observational
distribution constitute the unknown parameter. This representation
allows observations collected under one intervention to update reward
estimates for other interventions through their shared causal
mechanisms. We develop causal variants of Thompson Sampling and
Information-Directed Sampling (IDS) for this setting. For Thompson
Sampling, we establish an entropy-dependent sublinear Bayesian regret
bound. For IDS, we derive an entropy-dependent regret bound that
explicitly quantifies the additional error introduced by Monte Carlo
approximation of the expected regret and information gain; when these
quantities are available exactly, the bound recovers the standard
sublinear IDS rate. The dependence of these guarantees on
the action-set size is worst-case: our model contains the standard
multi-armed bandit as a special case. We further provide
high-probability confidence bounds for the Monte Carlo estimates.
Experiments on synthetic causal bandit tasks show that the
proposed methods outperform causal and non-causal baselines by
effectively exploiting information shared across interventions.
\end{abstract}

\section{Introduction}
\label{submission}

In the classical multi-armed bandit problem, a decision-maker repeatedly
selects an action from a finite set and observes the resulting reward.
Because the reward distributions are initially unknown, the learner must
balance \emph{exploration}, which gathers information about uncertain
actions, and \emph{exploitation}, which favors actions believed to yield
high rewards
\cite{lattimore2020bandit,slivkins2019introduction,vermorel2005multi}.
Without additional structure, observing the reward of one action
typically provides no direct information about the rewards of the other
actions. Consequently, effective algorithms must explore the available
actions sufficiently often
\cite{garivier2011kl,jamieson2014best,jamieson2014lil}. In many decision-making problems, however, actions are related through
a common underlying mechanism. An observation collected after one
action may therefore provide information about several other actions.
Structured bandit methods exploit such relationships to improve
statistical efficiency and reduce unnecessary exploration
\cite{schulz2020finding,jun2020crush,tirinzoni2020novel,
van2024optimal,wan2023towards,mersereau2009structured}.
The central challenge is to characterize and exploit the relevant
information-sharing structure while balancing exploration and
exploitation.

Causal bandits provide a principled framework for structured
decision-making when the environment is governed by an underlying
structural causal model
\cite{lattimore2016causal,sen2017best,lee2018structural,
wei2024approximate,qasim2024partial}. In this setting, each action
corresponds to an intervention on the causal system, and the resulting
observations are generated according to a shared causal model
\cite{pearl2009causality}. Unlike an unstructured bandit model, reward
distributions associated with different interventions are coupled
through common causal mechanisms. Consequently, observations obtained
under one intervention can improve the learner's estimates of the
rewards associated with other interventions
\cite{lattimore2016causal,sen2017best,yabe2018causal,
lee2018structural}.

Many existing causal bandit formulations permit interventions on a
large collection of observed variables. In practice, however, some
variables cannot be directly manipulated. Examples include genetic
characteristics in healthcare, demographic attributes in public-policy
applications, and macroeconomic conditions in economic decision-making.
Although such variables are non-manipulable, they may strongly influence
the reward and the effects of feasible interventions. Their presence
therefore changes the set of candidate interventions and creates an
additional challenge for efficiently sharing information across actions. Lee and Bareinboim~\cite{lee2019structural} study causal bandits with
non-manipulable variables in graphs that may contain latent confounding.
They characterize interventions that can be optimal under manipulability
constraints and use a generalized $z^2$-identification procedure to
derive multiple estimators of interventional reward distributions.
These estimators are combined through a bootstrap-based minimum-variance
weighted average and incorporated into variants of Thompson Sampling and
KL-UCB. Their results demonstrate the empirical value of exploiting
causal information, but do not provide regret guarantees for the
resulting algorithms.

Lu et al.~\cite{lu2020regret} study causal Thompson Sampling and UCB
under stronger prior knowledge, assuming that $P(\mathrm{Pa}_Y\mid a)$,
the distribution of the reward node's parents induced by intervention
$a$, is known to the learner. This allows their regret bound to depend
on the number of reward-parent configurations, yielding
$\widetilde{O}(\sqrt{k^nT})$ when $Y$ has $n$ parents with $k$ values
each, rather than directly on the number of interventions. In contrast,
our setting treats the underlying CPTs as unknown and learns them
sequentially, while also allowing contexts and non-manipulable
variables. Thus, their result illustrates sharper mechanism-dependent
guarantees available under stronger structural knowledge.

We study contextual causal bandits with non-manipulable variables under
the assumption that the causal graph is known and contains no latent
confounders. At each round, context variables are observed before the
learner selects an intervention. The learner subsequently observes the
reward and the remaining observed variables in the causal graph.
For example, in a healthcare application, patient characteristics and
medical history may be available before treatment selection, whereas
physiological measurements that cannot be directly manipulated may be
observed after treatment. The learner's objective is to select a
feasible intervention that maximizes the context-dependent expected
reward.

Following the Bayesian information-theoretic framework of \cite{russo2016information,russo2014learning}, we represent uncertainty about the causal system using a random parameter $\boldsymbol{\theta}$. Since the graph contains no latent confounders, interventional distributions are identifiable from the observational distribution via truncated factorization, and we let $\boldsymbol{\theta}$ collect the conditional probability tables of the causal graph. Each observation updates the posterior over these shared parameters, allowing data from one intervention to improve reward estimates and information gains for others. Based on this formulation, we develop causal variants of Thompson Sampling and Information-Directed Sampling. Thompson Sampling selects interventions according to their posterior probability of being optimal, whereas IDS balances expected instantaneous regret against information about the context-dependent optimal decision. Since the required posterior expectations are generally intractable, we estimate them using Monte Carlo samples and explicitly account for the approximation error. Our regret analysis uses a generic worst-case information-ratio bound, providing a safety-net guarantee: gains from causal sharing are instance-dependent, while sharper mechanism-dependent bounds require additional assumptions quantifying cross-intervention information sharing. Our main contributions are summarized as follows:

\begin{itemize}

\item We formulate contextual causal bandits with non-manipulable
variables in a Bayesian framework in which the conditional probability
tables of the observational distribution are treated as the unknown
parameter. This formulation enables observations collected under one
intervention to update estimates associated with other interventions
through their shared causal mechanisms.

\item We propose causal Thompson Sampling and Information-Directed
Sampling algorithms for this setting. We establish an entropy-dependent
sublinear Bayesian regret bound for Thompson Sampling. For IDS, we derive
a regret bound that separates the standard information-theoretic term
from the additional error caused by Monte Carlo approximation; the oracle
version, in which the information ratio is computed exactly, achieves
the standard sublinear IDS guarantee.

\item We evaluate the proposed methods on multiple structured synthetic
causal bandit tasks. The results demonstrate that our algorithms
outperform both causal and non-causal baselines by more effectively
exploiting information shared across interventions.

\end{itemize}

\section{Preliminaries}

We adopt the Structural Causal Model (SCM) framework \cite{pearl2009causality}. An SCM, denoted by $\mathcal{M}$, is defined as a 4-tuple $\langle \mathbf{U}, \mathbf{V}, \mathbf{F}, P(\mathbf{U}) \rangle$, where $\mathbf{U}$ denotes a set of exogenous (unobserved) variables, $\mathbf{V}$ denotes a set of endogenous (observed) variables, $\mathbf{F}$ denotes the corresponding structural functions, and $P(\mathbf{U})$ is the distribution over the exogenous variables. In our setting, $\mathbf{V}$ consists of the reward variable $Y$, manipulable variables, non-manipulable variables, and context variables. We assume that $Y$ is binary, while all other observed variables take values in arbitrary finite discrete sets. For any subset $\mathbf{X}\subseteq\mathbf{V}$, let $\Omega(\mathbf{X})$ denote the Cartesian product of the state spaces of the variables in $\mathbf{X}$. The structural functions specify how each $V_i$ is generated as $v_i=f_i(\mathbf{PA}^i,\mathbf{U}^i)$, where $\mathbf{PA}^i\subseteq\mathbf{V}$ and $\mathbf{U}^i\subseteq\mathbf{U}$ denote its observed and exogenous parents, respectively. 
The binary-reward assumption is used explicitly in
our information-ratio and information-gain concentration arguments and
is therefore essential to the present analysis. A finite-valued
extension can nevertheless be obtained through Bernoulliization: if
$Y\in\{0,\ldots,M\}$, after observing $Y$ we generate
$Y'\sim\operatorname{Bernoulli}(Y/M)$ and use $Y'$ as the reward
supplied to the algorithm. This preserves the ordering of expected
rewards and hence the optimal policy, while regret on the original
reward scale is obtained by multiplying the resulting binary-reward
bound by $M$.

Each SCM is associated with a causal graph $\mathcal{G}=\langle\mathbf{V},\mathbf{E}\rangle$. A directed edge $V_i\rightarrow V_j$ represents direct functional dependence, whereas a bidirected edge $V_i\leftrightarrow V_j$ represents latent confounding between $V_i$ and $V_j$. We use $\textit{pa}$, $\textit{ch}$, $\textit{an}$, and $\textit{de}$ to denote parents, children, ancestors, and descendants, respectively, and their capitalized forms $\textit{Pa}$, $\textit{Ch}$, $\textit{An}$, and $\textit{De}$ to include the variable itself; for example, $\textit{An}(W)=\textit{an}(W)\cup\{W\}$. For a set of variables these operators are defined by union, e.g., $\textit{An}(\mathbf{W})=\bigcup_{W\in\mathbf{W}}\textit{An}(W)$, and $\textit{pa}(V_i)=\mathbf{PA}^i$. The graph $\mathcal{G}_{\overline{\mathbf{X}}}$ is obtained by removing all directed edges entering the variables in $\mathbf{X}$. The c-component containing $V_i$, denoted by $\mathsf{CC}_{\mathcal{G}}(V_i)$, is the maximal set of vertices connected to $V_i$ by paths consisting only of bidirected edges \cite{tian2002general}.

In the $K$-armed bandit problem, the learner repeatedly selects among $K$ actions with the objective of minimizing cumulative regret over $T$ rounds. In the SCM-MAB setting, each arm corresponds to an intervention $\mathrm{do}(\mathbf{X}=\mathbf{x})$, where $\mathbf{X}$ is a subset of the designated manipulable variables. Context variables are observed before action selection and are not intervention targets, but are treated separately from the non-manipulable variables $\mathbf{N}$. The reward distribution under $\mathrm{do}(\mathbf{X}=\mathbf{x})$, denoted by $P(Y_{\mathbf{x}})$, coincides with the interventional distribution $P_{\mathbf{x}}(Y)$, with expected reward $\mu_{\mathbf{x}}=\mathbb{E}[Y\mid\mathrm{do}(\mathbf{x})]$. When context variables $\mathbf{C}$ are observed prior to intervention, the relevant expected reward is $\mu_{\mathbf{x}}(\mathbf{C}=\mathbf{c})=\mathbb{E}[Y\mid\mathrm{do}(\mathbf{x}),\mathbf{C}=\mathbf{c}]$. We assume that $\mathbf{C}$ is closed under ancestry, i.e., $\textit{An}(\mathbf{C})=\mathbf{C}$, so that context variables are unaffected by feasible interventions and retain their interpretation as pre-intervention information.

\section{Possibly Optimal Arms for Causal Bandits}

In this section, we revisit the results of \cite{lee2019structural}, which characterize possibly optimal arms in causal bandits when certain nodes in the causal graph are non-manipulable. Let \( \mathbf{N} \subseteq \mathbf{V} \setminus \{Y\} \) denote the set of non-manipulable variables, noting that the reward variable \( Y \) is also inherently non-manipulable.
 
\begin{definition} (\textbf{Unobserved Confounder (UC)-Territory} \cite{lee2018structural}) Consider a causal graph $\mathcal{G}(\mathbf{V},\mathbf{E})$ with reward node $Y$, and let \( \mathcal{H} \) be the subgraph \( \mathcal{G}[\mathsf{An}(Y)] \). A set of variables \( \mathbf{T} \subseteq V(\mathcal{H}) \) containing \( Y \) is called a \textbf{UC-territory} on \( \mathcal{G} \) with respect to \( Y \) if \( \mathsf{De}_{\mathcal{H}}(\mathbf{T})=\mathbf{T} \) and \( \mathsf{CC}_{\mathcal{H}}(\mathbf{T})=\mathbf{T} \).
\end{definition} 

A UC-territory is minimal if none of its proper subsets containing $Y$ is a UC-territory. A minimal UC-territory, denoted by $\mathsf{MUCT}(\mathcal{G},Y)$, can be constructed by starting from the set $\{Y\}$ and alternately extending the current set by its c-component and its descendants until the set no longer changes.

\begin{definition} \textbf{(Interventional Border \cite{lee2018structural})}
Let \( \mathbf{T} \) be a minimal UC-territory on \( \mathcal{G} \) with respect to \( Y \). Then, \( \mathbf{X}=\mathsf{Pa}(\mathbf{T})\setminus\mathbf{T} \) is called the \textit{interventional border} of \( \mathcal{G} \) with respect to \( Y \), denoted by $\mathsf{IB}(\mathcal{G},Y)$.
\end{definition}

\begin{lemma} \cite{lee2018structural} For a graph $\mathcal{G}$ with reward node $Y$, $\mathsf{IB}(\mathcal{G}_{\overline{\mathbf{W}}},Y)$ is a possibly optimal minimal intervention set (POMIS) for any \( \mathbf{W}\subseteq\mathbf{V}\setminus\{Y\} \).
\label{lemm_pomis_o}
\end{lemma}

\begin{figure}[h]
\centering

\begin{minipage}{0.23\linewidth}
\centering
\begin{tikzpicture}[scale=0.75, transform shape=false]
    \node (1) {$X$};
    \node (2) [right=of 1] {$Y$};
    \node[red] (4) [above right=of 1,xshift=-0.6cm,yshift=-0.1cm] {$Z$};

    \path (1) edge (2);
    \path (4) edge (1);
    \path (4) edge (2);

    \node at (0.8,-0.9) {\normalsize (a) $\mathcal{G}_1$};
\end{tikzpicture}
\end{minipage}
\hfill
\begin{minipage}{0.23\linewidth}
\centering
\begin{tikzpicture}[scale=0.75, transform shape=false]
    \node (1) {$X$};
    \node (2) [right=of 1] {$Y$};

    \path (1) edge (2);
    \path[bidirected] (1) edge[bend left=70] (2);

    \node at (0.8,-0.9) {\normalsize (b) $\mathcal{H}_1$};
\end{tikzpicture}
\end{minipage}
\hfill
\begin{minipage}{0.25\linewidth}
\centering
\begin{tikzpicture}[scale=0.70, transform shape=false]
    \node (1) {$X$};
    \node (2) [right=of 1] {$W$};
    \node (3) [right=of 2] {$Y$};

    \node[red] (5) [above right=of 1,xshift=-0.6cm,yshift=-0.1cm] {$Z_1$};
    \node[red] (6) [right=of 5] {$Z_2$};

    \path (1) edge (2);
    \path (2) edge (3);
    \path (5) edge (1);
    \path (5) edge (6);
    \path (6) edge (3);

    \node at (1.6,-0.9) {\normalsize (c) $\mathcal{G}_2$};
\end{tikzpicture}
\end{minipage}
\hfill
\begin{minipage}{0.25\linewidth}
\centering
\begin{tikzpicture}[scale=0.70, transform shape=false]
    \node (1) {$X$};
    \node (2) [right=of 1] {$W$};
    \node (3) [right=of 2] {$Y$};

    \path (1) edge (2);
    \path (2) edge (3);
    \path[bidirected] (1) edge[bend left=70] (3);

    \node at (1.6,-0.9) {\normalsize (d) $\mathcal{H}_2$};
\end{tikzpicture}
\end{minipage}

\caption{Original causal graphs $\mathcal{G}_1$ and $\mathcal{G}_2$, along with their projections $\mathcal{H}_1$ and $\mathcal{H}_2$.}
\label{fig1}
\end{figure}
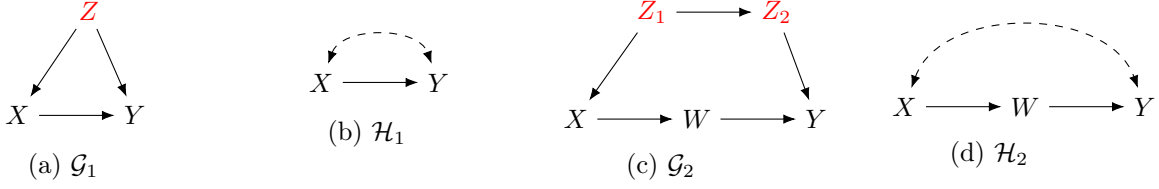

\vspace{-0.5em}

For a causal graph $\mathcal{G}$ in which all nodes except the reward node $Y$ are manipulable, we use $\mathbb{P}_{\mathcal{G},Y}$ to denote the collection of all POMISs. In the more general setting where variables in $\mathbf{N}$ cannot be intervened upon, we use $\mathbb{P}^{\mathbf{N}}_{\mathcal{G},Y}$ to denote the corresponding collection of possibly optimal minimal intervention sets. Following \cite{lee2019structural}, these sets are obtained through a projection step. We initialize $\mathcal{H}=\langle\mathbf{V}\setminus\mathbf{N},\emptyset\rangle$ and add a directed edge $V_i\to V_j$ if $V_i\to V_j$ is present in $\mathcal{G}$, or if there exists a directed path from $V_i$ to $V_j$ whose non-endpoint vertices lie in $\mathbf{N}$. We add a bidirected edge $V_i\leftrightarrow V_j$ if it is present in $\mathcal{G}$, or if projecting paths through $\mathbf{N}$ induces latent confounding between $V_i$ and $V_j$. The results of \cite{lee2019structural} imply that the POMISs under these constraints can be obtained by enumerating those of $\mathcal{H}$ using Lemma~\ref{lemm_pomis_o}, so $\mathbb{P}^{\mathbf{N}}_{\mathcal{G},Y}=\mathbb{P}_{\mathcal{H},Y}$. In the contextual setting, context variables are treated separately from $\mathbf{N}$. Since $\mathbf{C}_t$ is observed before action selection and satisfies $\operatorname{An}(\mathbf{C}_t)=\mathbf{C}_t$, feasible interventions cannot affect its realized value. Thus, after observing $\mathbf{C}_t=\mathbf{c}_t$, the context acts as fixed pre-intervention information rather than an unresolved non-manipulable variable. The projection therefore accounts for $\mathbf{N}$, while intervention rewards are evaluated as $\mathbb{E}[Y\mid do(\mathbf{a}),\mathbf{C}_t=\mathbf{c}_t]$; manipulability determines the feasible intervention set, whereas the realized context can change the rewards and optimal intervention.

Consider the causal graph $\mathcal{G}_1$ with reward node $Y$ and non-manipulable set $\mathbf{N}=\{Z\}$, and let $\mathcal{H}_1$ be the graph resulting from the projection step, as shown in Figure~\ref{fig1}. The possibly optimal intervention sets for the original graph can be obtained from $\mathcal{H}_1$ using Lemma~\ref{lemm_pomis_o}. In particular, $\mathbb{P}_{\mathcal{H}_1,Y}=\{\emptyset,\{X\}\}=\mathbb{P}^{\mathbf{N}}_{\mathcal{G}_1,Y}$, so the candidate optimal arms are $do()$, $do(X=0)$, and $do(X=1)$. Similarly, for $\mathcal{G}_2$ with reward node $Y$ and $\mathbf{N}=\{Z_1,Z_2\}$, the projected graph $\mathcal{H}_2$ is shown in Figure~\ref{fig1}. In this case, $\mathbb{P}_{\mathcal{H}_2,Y}=\{\emptyset,\{W\}\}=\mathbb{P}^{\mathbf{N}}_{\mathcal{G}_2,Y}$, and the candidate optimal arms are $do()$, $do(W=0)$, and $do(W=1)$. More generally, $\mathcal{A}:=\bigcup_{\mathbf{S}\in\mathbb{P}^{\mathbf{N}}_{\mathcal{G},Y}}\{do(\mathbf{S}=\mathbf{s}):\mathbf{s}\in\Omega(\mathbf{S})\}$ denotes the set of interventions generated by all POMISs, where $do(\emptyset)$ denotes the observational action.

We use $\mathcal{G}_1$ to illustrate how observations from different arms can improve reward estimates across interventions. Let $Z$ denote patient age, $X$ a treatment, and $Y$ the health outcome. Since age affects the outcome but cannot be manipulated, the candidate arms are $do()$, $do(X=0)$, and $do(X=1)$, with expected rewards $P(Y=1\mid do())=\sum_{x,z}P(Y=1\mid x,z)P(x\mid z)P(z)$, $P(Y=1\mid do(X=0))=\sum_zP(Y=1\mid X=0,z)P(z)$, and $P(Y=1\mid do(X=1))=\sum_zP(Y=1\mid X=1,z)P(z)$. The graph is known but the observational distribution is unknown. Since all non-intervened variables are observed, data from different interventions can be combined to estimate shared conditional distributions, allowing observations from one arm to improve reward estimates for others. We formalize this sharing in the next section through a Bayesian posterior over the observational-distribution parameters.

\section{Bayesian Formulation for Causal Bandits}

We consider a general probabilistic (Bayesian) formulation in which uncertain quantities are modeled as random variables. At each time $t$, the agent selects an action $A_t \in \mathcal{A}$ and observes a reward. For each action $a \in \mathcal{A}$ and time $t$, let $Y_{a,t} \in \mathcal{Y}$ denote the potential reward that would be observed at time $t$ if action $a$ were selected. There exists an unknown real-valued parameter vector $\boldsymbol{\theta}$ such that, conditioned on $\boldsymbol{\theta}$, the rewards are independent across time and satisfy $\mathbb{E}[Y_{A_t,t} \mid \boldsymbol{\theta}] = \mu_{\boldsymbol{\theta}}(A_t)$. The parameter $\boldsymbol{\theta}$ is random under a prior distribution, capturing uncertainty about the underlying reward-generating mechanism. Uncertainty in $\boldsymbol{\theta}$ induces uncertainty about the optimal action $A^\star \in \arg\max_{a \in \mathcal{A}} \mu_{\boldsymbol{\theta}}(a)$, where ties are broken using a fixed deterministic rule. The mean cumulative regret of a policy is defined as 
$\mathbb{E}[\mathrm{Reg}_T] = \mathbb{E}\!\left[\sum_{t=1}^T \left(Y_{A^\star,t} - Y_{A_t,t}\right)\right]$, 
where the expectation is taken over the randomness in the actions $A_t$, the rewards, and the prior distribution over $\boldsymbol{\theta}$. This performance metric is commonly referred to as Bayesian regret or Bayesian risk.

Consider the history of observations $\mathcal{F}_t=\left(\mathbf{C}_1,A_1,\mathbf{V}_{A_1,1},Y_{A_1,1},\ldots,\mathbf{C}_{t-1},A_{t-1},\mathbf{V}_{A_{t-1},t-1},Y_{A_{t-1},t-1}\right)$ available before observing the context at time $t$. The action $A_t$ is selected using the history of observations $\mathcal{F}_t$. Formally, a randomized policy $\pi=\left(\pi_t\right)_{t \in \mathbb{N}}$ is a sequence of deterministic functions, where $\pi_t\left(\mathcal{F}_t,\mathbf{C}_t\right)$ specifies a probability distribution over the action set $\mathcal{A}$. Let $\mathcal{D}(\mathcal{A})$ denote the set of probability distributions over $\mathcal{A}$. After observing $\mathbf{C}_t$, the action $A_t$ is selected by sampling from $\pi_t\left(\mathcal{F}_t,\mathbf{C}_t\right)$. With some abuse of notation, we typically denote this distribution by $\pi_t$, where $\pi_t(a \mid \mathbf{c}_t)=\mathbb{P}\left(A_t=a \mid \mathcal{F}_t,\mathbf{C}_t=\mathbf{c}_t\right)$ denotes the probability assigned to action $a$ under the realized context $\mathbf{c}_t$. As shorthand notation, we use \( P_t(\cdot) \) for \( P(\cdot \mid \mathcal{F}_t) \) and \( \mathbb{E}_t(\cdot) \) for \( \mathbb{E}(\cdot \mid \mathcal{F}_t) \), respectively. After observing $\mathbf{C}_t=\mathbf{c}_t$, we define the context-conditioned posterior by \(P_t^{\mathbf{c}_t}(\cdot):=P(\cdot \mid \mathcal{F}_t,\mathbf{C}_t=\mathbf{c}_t)\) and the corresponding conditional expectation by \(\mathbb{E}_t^{\mathbf{c}_t}(\cdot)\).

In the contextual setting, at each time $t$, a context $\mathbf{C}_t \in \mathcal{C}$ is observed prior to selecting an action. The expected reward depends on both the action and the context and satisfies $\mathbb{E}[Y_{A_t,t} \mid \mathbf{C}_t=\mathbf{c},\boldsymbol{\theta}] = \mu_{\boldsymbol{\theta}}(A_t,\mathbf{c})$. Accordingly, we define the oracle mapping $\Pi^\star:\mathcal{C}\to\mathcal{A}$ by $\Pi^\star(\mathbf{c}):=\arg\max_{a \in \mathcal{A}}\mu_{\boldsymbol{\theta}}(a,\mathbf{c})$, where ties are broken using a fixed deterministic rule. 
The mean cumulative Bayesian regret relative to this oracle mapping is
defined sequentially as
\begin{equation}
\mathbb{E}\!\left[
\sum_{t=1}^{T}
\left(
Y_{\Pi^\star(\mathbf{C}_t),t}
-
Y_{A_t,t}
\right)
\right],
\label{eq:contextual-bayesian-regret}
\end{equation}
where the expectation is over the prior, contexts, rewards, and
action-selection randomness. At round $t$, the learner conditions only on
the history $\mathcal{F}_t$ and the currently observed context
$\mathbf{C}_t$. Hence, the contexts need not be exogenous and may
themselves provide information about the unknown parameter
$\boldsymbol{\theta}$.

One remaining task is to specify the parameter $\boldsymbol{\theta}$ in the causal bandit setting. We assume that the causal graph contains no unobserved confounders. Under this assumption, any interventional distribution, including the conditional interventional distributions arising in the contextual setting, is identifiable from the observational distribution and the graph structure through the truncated factorization formula \cite{pearl2009causality}. For instance, $P(\mathbf{v}_{\mathbf{V}\setminus\mathbf{X}} \mid do(\mathbf{x})) = \prod_{V_i \in \mathbf{V} \setminus \mathbf{X}} P(v_i \mid \mathbf{pa}_i)$, where the parent values are evaluated consistently with the intervention $\mathbf{X}=\mathbf{x}$. Thus, the observational distribution naturally determines the parameter $\boldsymbol{\theta}$.

Under the Markov assumption, the joint distribution factorizes as $P(\mathbf{v}) = \prod_{V_i \in \mathbf{V}} P(v_i \mid \mathbf{pa}_i)$. Rather than computing the posterior over the entire joint observational distribution, we compute the posterior over its local conditional probability distributions using the history of observations $\mathcal{F}_t$. For each variable $V_i \in \mathbf{V}$ with $r_i$ possible values $x_i^1,\ldots,x_i^{r_i}$, the local likelihood consists of a collection of multinomial distributions, one for each configuration of $\mathbf{Pa}_i$. Specifically, $P\left(V_i=x_i^k \mid \mathbf{Pa}_i=\mathbf{pa}_i^j,\boldsymbol{\theta}_{ij}\right)=\theta_{ijk}$, where $\theta_{ijk}>0$ and $\sum_{k=1}^{r_i}\theta_{ijk}=1$. Here, $\mathbf{pa}_i^1,\ldots,\mathbf{pa}_i^{q_i}$ denote the $q_i$ possible configurations of the parent set $\mathbf{Pa}_i$, where $q_i=\prod_{X_\ell \in \mathbf{Pa}_i}r_\ell$, and $\boldsymbol{\theta}_{ij}=(\theta_{ij1},\ldots,\theta_{ijr_i})$ denotes the parameter vector associated with the $j$-th parent configuration. We define $\boldsymbol{\theta}_i=(\boldsymbol{\theta}_{i1},\ldots,\boldsymbol{\theta}_{iq_i})$, and the overall parameter vector is therefore $\boldsymbol{\theta}=(\boldsymbol{\theta}_1,\ldots,\boldsymbol{\theta}_n)$.

Similar to \cite{heckerman2006bayesian}, to ensure efficient computation, we assume parameter independence across the vectors $\boldsymbol{\theta}_{ij}$. This is a modeling choice for computational tractability, rather than a consequence of the Markov factorization. Consequently, conditioned on the history of observations $\mathcal{F}_t$, each vector $\boldsymbol{\theta}_{ij}$ can be updated independently. Assuming that each vector $\boldsymbol{\theta}_{ij}$ follows a conjugate Dirichlet prior, namely $\operatorname{Dir}\!\left(\boldsymbol{\theta}_{ij} \mid \alpha_{ij1},\ldots,\alpha_{ijr_i}\right)$, the posterior distribution over the parameter space is given by

\begin{equation}
P_t(\boldsymbol{\theta}) = P\left(\boldsymbol{\theta} \mid \mathcal{F}_t \right) = \prod_{i=1}^n \prod_{j=1}^{q_i} P_t(\boldsymbol{\theta}_{ij}).
\label{posterior_comp}
\end{equation}
\begin{equation}
P_t(\boldsymbol{\theta}_{ij}) = \operatorname{Dir}\left(\boldsymbol{\theta}_{ij} \mid \alpha_{ij1}+N_{ij1}(\mathcal{F}_t),\ldots,\alpha_{ijr_i}+N_{ijr_i}(\mathcal{F}_t)\right).
\label{eq_2}
\end{equation}

Here, $N_{ijk}(\mathcal{F}_t)$ denotes the number of times the realization $V_i=x_i^k$ has been observed in the history $\mathcal{F}_t$ when $\mathbf{Pa}_i=\mathbf{pa}_i^j$ and the corresponding action does not intervene on $V_i$. After observing the current context $\mathbf{C}_t=\mathbf{c}_t$, the posterior used for action selection is $P_t^{\mathbf{c}_t}(\boldsymbol{\theta})=P(\boldsymbol{\theta}\mid\mathcal{F}_t,\mathbf{C}_t=\mathbf{c}_t)$, which is obtained by conditioning $P_t(\boldsymbol{\theta})$ on the observed context. Equation~\eqref{posterior_comp} can be used to compute the posterior over the parameter space and to draw posterior samples. Maintaining an accurate posterior is crucial for effective exploration and exploitation in the bandit problem.

\vspace{-0.5em}

\section{Thompson Sampling Algorithm for Causal bandits}

A bandit algorithm can select actions according to the posterior probability that an intervention is optimal, a strategy known as posterior sampling or Thompson sampling \cite{thompson1933likelihood}. We propose Thompson sampling for causal bandits with non-manipulable variables. Algorithm~\ref{alg: thom} maintains a posterior $P_t(\boldsymbol{\theta})$ over the parameter space at each round $t$. After observing $\mathbf{C}_t=\mathbf{c}_t$, the action is selected as
$
a_t \in \arg\max_{a \in \mathcal{A}} \mathbb{E}[Y_{a,t} \mid \boldsymbol{\theta}_t,\mathbf{c}_t],
$
where $\boldsymbol{\theta}_t$ is sampled from the context-conditioned posterior $P_t(\boldsymbol{\theta}\mid\mathbf{C}_t=\mathbf{c}_t)$. The posterior is updated after each interaction. Our contextual analysis follows the standard posterior-sampling argument of \cite{russo2016information}, conditioned on the observed context. We introduce a nonnegative coefficient $\Gamma_t$ such that the expected instantaneous regret satisfies
\begin{equation*}
\mathbb{E}_t\!\left[
Y_{\Pi^\star(\mathbf{C}_t),t} - Y_{A_t,t}
\mid \mathbf{C}_t
\right]
\leq
\Gamma_t \sqrt{
I_t\!\left(
\Pi^\star(\mathbf{C}_t);
(A_t,Y_{A_t,t})
\mid \mathbf{C}_t
\right)
}.
\end{equation*}
The quantity $\Gamma_t^2$ is referred to as the information ratio and characterizes the trade-off between incurring low regret and acquiring information about the optimal action under the realized context. The term
$
I_t\!\left(
\Pi^\star(\mathbf{C}_t);
(A_t,Y_{A_t,t})
\mid \mathbf{C}_t
\right)
:=
D_{\mathrm{KL}}\!\left(
P_t\!\left(
\Pi^\star(\mathbf{C}_t),A_t,Y_{A_t,t}
\mid \mathbf{C}_t
\right)
\;\middle\|\;
P_t\!\left(
\Pi^\star(\mathbf{C}_t)
\mid \mathbf{C}_t
\right)
P_t\!\left(
A_t,Y_{A_t,t}
\mid \mathbf{C}_t
\right)
\right)
$
denotes the filtered mutual information between the optimal action under the realized context and the action-reward observation at round $t$. This quantity measures how much the observation collected at round $t$ reduces uncertainty about the optimal action under the realized context. Since $\Pi^\star(\mathbf{C}_t)$ is a deterministic function of the full oracle mapping $\Pi^\star$, the cumulative information gain is bounded by the entropy $H(\Pi^\star)$. Theorem~\ref{thm1} provides a bound on the mean cumulative regret.

\begin{theorem}
The mean cumulative Bayesian regret of Algorithm~\ref{alg: thom} is bounded for any $T \in \mathbb{N}$ as follows, where $\Gamma \geq \Gamma_t$ for all $t \in \{1,\ldots,T\}$:
\begin{equation*}
\mathbb{E}\!\left[
\sum_{t=1}^T \left(
Y_{\Pi^\star(\mathbf{C}_t),t} - Y_{A_t,t}
\right)
\right]
\leq
\Gamma \sqrt{H(\Pi^\star) \, T}.
\end{equation*}
\label{thm1}
\end{theorem}

\vspace{-1em}
The entropy $H(\Pi^\star)$ quantifies the decision-maker's initial uncertainty regarding the optimal-action mapping. The mean cumulative regret incurred by Algorithm~\ref{alg: thom} depends on the time horizon $T$, the entropy $H(\Pi^\star)$, and the worst-case upper bound on the information-ratio coefficient $\Gamma_t$. Following \cite{russo2016information}, we establish the worst-case bound $\Gamma_t^2 \leq \frac{|\mathcal{A}|}{2}$, or equivalently $\Gamma_t \leq \sqrt{\frac{|\mathcal{A}|}{2}}$. 
This dependence on $|\mathcal{A}|$ cannot in general be removed without additional assumptions on mechanism sharing. To see this, consider the causal graph $X\rightarrow Y$, where $Y$ is binary and $X$ is a manipulable $K$-valued variable, with $\mathcal{A}=\{do(X=j):j\in[K]\}$. Let $\theta_j=P(Y=1\mid X=j)$, with independent uncertainty across the $K$ reward-CPT rows. Under $do(X=j)$, the reward is $\operatorname{Bernoulli}(\theta_j)$ and provides no information about $\theta_\ell$ for $\ell\neq j$, reducing the problem exactly to a standard $K$-armed Bernoulli bandit. Taking one unknown arm with mean $1/2+\epsilon$, the others with mean $1/2$, and $\epsilon=\Theta(\sqrt{K/T})$ gives the standard $\Omega(\sqrt{KT})$ worst-case lower bound for $K=|\mathcal{A}|$. Thus, our $|\mathcal{A}|$ dependence is a generic worst-case guarantee rather than a structure-optimal one. In comparison, Lu et al.~\cite{lu2020regret} assume that the intervention-induced distribution $P(\mathrm{Pa}_Y\mid a)$ is known for every action $a$, so only the reward mechanism over the configurations of $\mathrm{Pa}_Y$ must be learned; when $Y$ has $n$ parents with $k$ values each, this yields a mechanism-dependent $\widetilde{O}(\sqrt{k^nT})$ Bayesian-regret bound for causal Thompson Sampling. Our setting does not assume these intervention-specific parent distributions or the underlying CPTs are known, and instead learns the causal mechanisms sequentially, so the two guarantees apply under different levels of prior structural knowledge.

Therefore, the constant $\Gamma$ in Theorem~\ref{thm1} can be replaced with $\sqrt{\frac{|\mathcal{A}|}{2}}$, leading to the regret bound stated in Corollary~\ref{crr_reg}. Moreover, since $\Pi^\star$ maps the context space $\mathcal{C}$ to the action space $\mathcal{A}$, the number of possible mappings is $|\mathcal{A}|^{|\mathcal{C}|}$ when $\mathcal{C}$ is finite, which implies $H(\Pi^\star) \leq |\mathcal{C}|\log|\mathcal{A}|$.

\begin{algorithm}[t]
\small
\SetKwInOut{Input}{Input}

\Input{Causal graph $\mathcal{G}$ and the action set $\mathcal{A}$ induced by all POMISs of the projected graph.}

\For{$t = 1$ {\bf to} $T$}{
    Compute the posterior $P_t(\boldsymbol{\theta})$ using \eqref{posterior_comp}\;
    Observe the realized context $\mathbf{C}_t=\mathbf{c}_t$\;
    Compute the context-conditioned posterior $P_t^{\mathbf{c}_t}(\boldsymbol{\theta})$\;
    Sample $\boldsymbol{\theta}_t \sim P_t^{\mathbf{c}_t}(\boldsymbol{\theta})$\;
    Select action 
    $a_t \leftarrow \arg\max_{a \in \mathcal{A}} 
    \mathbb{E}[Y_{a,t} \mid \boldsymbol{\theta}_t,\mathbf{c}_t]$\;
    Play action $a_t$, observe the reward and all non-intervened variables, and update the history $\mathcal{F}_{t+1}$\;
}

\caption{Thompson Sampling for causal bandits with non-manipulable variables}
\label{alg: thom}
\end{algorithm}

\begin{corollary}
Suppose that the context space $\mathcal{C}$ is finite. The mean cumulative Bayesian regret for Algorithm~\ref{alg: thom} is bounded for any $T \in \mathbb{N}$ as follows:
\begin{equation*}
\mathbb{E}\!\left[
\sum_{t=1}^T \left(
Y_{\Pi^\star(\mathbf{C}_t),t} - Y_{A_t,t}
\right)
\right]
\leq
\sqrt{\frac{T|\mathcal{A}| |\mathcal{C}|\log\!|\mathcal{A}|}{2}}.
\end{equation*}
\label{crr_reg}
\end{corollary}

\section{Information-Directed Sampling Algorithm for Causal Bandits}

Information-Directed Sampling (IDS) selects actions by balancing immediate regret against information acquisition. In the contextual setting, after observing $\mathbf{C}_t=\mathbf{c}_t$, IDS chooses a distribution $\pi\in\mathcal{D}(\mathcal{A})$ that minimizes the squared expected instantaneous regret relative to the information gained about the context-dependent optimal action $\Pi^\star(\mathbf{c}_t)$. The resulting IDS policy is
\begin{equation*}
\pi_t^{\mathrm{IDS}}(\mathbf{c}_t)
\in
\arg\min_{\pi\in\mathcal{D}(\mathcal{A})}
\left\{
\Psi_t(\pi\mid\mathbf{c}_t)
:=
\frac{\Delta_t(\pi\mid\mathbf{c}_t)^2}
{g_t(\pi\mid\mathbf{c}_t)}
\right\}.
\end{equation*}
The expected instantaneous regret of a policy $\pi$ is
$\Delta_t(\pi\mid\mathbf{c}_t)=\sum_{a\in\mathcal{A}}\pi(a)\Delta_t(a\mid\mathbf{c}_t)$, where
$\Delta_t(a\mid\mathbf{c}_t):=\mathbb{E}_t\!\left[Y_{\Pi^\star(\mathbf{c}_t),t}-Y_{a,t}\mid\mathbf{C}_t=\mathbf{c}_t\right]$.
Similarly, its information gain is
$g_t(\pi\mid\mathbf{c}_t)=\sum_{a\in\mathcal{A}}\pi(a)g_t(a\mid\mathbf{c}_t)$, where
$g_t(a\mid\mathbf{c}_t):=I_t\!\left(\Pi^\star(\mathbf{c}_t);Y_{a,t}\mid\mathbf{C}_t=\mathbf{c}_t\right)$.
Equivalently,
$g_t(a\mid\mathbf{c}_t)=H_t\!\left(\Pi^\star(\mathbf{c}_t)\mid\mathbf{C}_t=\mathbf{c}_t\right)-H_t\!\left(\Pi^\star(\mathbf{c}_t)\mid\mathbf{C}_t=\mathbf{c}_t,Y_{a,t}\right)$,
so it measures the expected reduction in posterior uncertainty about the optimal action after observing $Y_{a,t}$. For the information ratio, we adopt the conventions $0/0:=0$ and $x/0:=+\infty$ for every $x>0$. Since the posterior expectations defining $\Delta_t$ and $g_t$ are generally unavailable in closed form, Algorithm~\ref{alg: ids} evaluates the information ratio using Monte Carlo approximations of these quantities.

\begin{algorithm}[h]
\caption{Information-Directed Sampling for Causal Bandits with Non-Manipulable Variables}
\label{alg: ids}

\SetKwInput{Input}{Input}
\SetKwInput{Initialization}{Initialization}

\Input{Causal graph $\mathcal{G}$, confidence level $\delta'$, posterior sample size $N$, and action set $\mathcal{A}$ induced by all POMISs of the projected graph}

\For{$t = 1$ \KwTo $T$}{

Compute the posterior $P_t(\boldsymbol{\theta})$ using \eqref{posterior_comp}\;

Observe the realized context $\mathbf{C}_t=\mathbf{c}_t$\;

Compute the context-conditioned posterior $P_t^{\mathbf{c}_t}(\boldsymbol{\theta})=P(\boldsymbol{\theta}\mid\mathcal{F}_t,\mathbf{C}_t=\mathbf{c}_t)$\;

Draw $N$ independent samples from $P_t^{\mathbf{c}_t}(\boldsymbol{\theta})$\;

For every $a\in\mathcal{A}$, select any value $\widetilde{\Delta}_t(a\mid\mathbf{c}_t)$ from the confidence interval in Lemma~\ref{lemm_2}, using $\delta=\frac{\delta'}{2T|\mathcal{A}|}$, and project it onto $[0,1]$ as
$\widetilde{\Delta}_t(a\mid\mathbf{c}_t)\leftarrow\min\{1,\max\{0,\widetilde{\Delta}_t(a\mid\mathbf{c}_t)\}\}$\;

For every $a\in\mathcal{A}$, select any value $\widetilde{g}_t(a\mid\mathbf{c}_t)$ from the confidence interval
$\left[\underline{g}_t(a\mid\mathbf{c}_t),\overline{g}_t(a\mid\mathbf{c}_t)\right]$
in Lemma~\ref{lemm_gt_conc}, using $\delta=\frac{\delta'}{2T|\mathcal{A}|}$\;

Set $\vec{\Delta}_t(\mathbf{c}_t):=[\widetilde{\Delta}_t(a\mid\mathbf{c}_t)]_{a\in\mathcal{A}}$
and $\vec{g}_t(\mathbf{c}_t):=[\widetilde{g}_t(a\mid\mathbf{c}_t)]_{a\in\mathcal{A}}$

$a_t\leftarrow\text{IDSAction}\!\left(\mathcal{A},
\vec{\Delta}_t(\mathbf{c}_t),
\vec{g}_t(\mathbf{c}_t)\right)$\;

Play action $a_t$, observe the reward and all non-intervened variables, and update the history $\mathcal{F}_{t+1}$\;
}
\end{algorithm}

We use an example to illustrate why the quantity $g_t(a\mid\mathbf{c}_t)$ is useful for characterizing how informative a particular action is. Since we treat the observational distribution as defining the parameter space, one might argue that the empty intervention $do()$ is the most informative about the parameter $\boldsymbol{\theta}$, as it provides direct samples from the natural data-generating process. However, our objective is not merely to learn the full observational distribution, but rather to reduce uncertainty about the optimal action under the realized context $\mathbf{C}_t=\mathbf{c}_t$. Consider a simple causal graph in which $X\rightarrow Y$. The quantities relevant for identifying the optimal action are the context-dependent conditional probabilities $P(Y=1\mid X=0,\mathbf{c}_t)$ and $P(Y=1\mid X=1,\mathbf{c}_t)$. If $P(X=0\mid\mathbf{c}_t)$ is small relative to $P(X=1\mid\mathbf{c}_t)$, then samples drawn from the observational distribution will rarely provide information about the behavior of $Y$ when $X=0$ under that context. In such a case, performing the intervention $do(X=0)$ may be substantially more informative than $do()$ for determining the optimal action $\Pi^\star(\mathbf{c}_t)$. This distinction is precisely captured by $g_t(a\mid\mathbf{c}_t)$, which measures the expected reduction in posterior uncertainty about $\Pi^\star(\mathbf{c}_t)$ resulting from selecting action $a$. We now proceed to formally derive the expressions needed to compute $\Delta_t(a\mid\mathbf{c}_t)$ and $g_t(a\mid\mathbf{c}_t)$.

\begin{lemma}
\label{lemm_1}
The information gain associated with an action $a\in\mathcal{A}$ at
time step $t$, conditional on the realized context
$\mathbf{C}_t=\mathbf{c}_t$, is given by
\begin{equation}
g_t(a\mid\mathbf{c}_t)
=
D_{\mathrm{KL}}\!\left(
P_t^{\mathbf{c}_t}\!\left(
\Pi^\star(\mathbf{c}_t),Y_a
\right)
\,\middle\|\,
P_t^{\mathbf{c}_t}\!\left(
\Pi^\star(\mathbf{c}_t)
\right)
P_t^{\mathbf{c}_t}\!\left(
Y_a
\right)
\right).
\end{equation}
The corresponding posterior distributions satisfy
\begin{equation}
P_t^{\mathbf{c}_t}\!\left(
\Pi^\star(\mathbf{c}_t)=a^\star
\right)
=
\mathbb{E}_{\boldsymbol{\theta}\sim P_t^{\mathbf{c}_t}(\boldsymbol{\theta})}
\left[
\mathbbm{1}\!\left\{
a^\star
=
\arg\max_{a'\in\mathcal{A}}
\mu_{\boldsymbol{\theta}}(a',\mathbf{c}_t)
\right\}
\right],
\end{equation}
\begin{equation}
P_t^{\mathbf{c}_t}\!\left(
Y_a=y
\right)
=
\mathbb{E}_{\boldsymbol{\theta}\sim P_t^{\mathbf{c}_t}(\boldsymbol{\theta})}
\left[
P\!\left(
Y_a=y
\mid\boldsymbol{\theta},\mathbf{c}_t
\right)
\right],
\end{equation}
and
\begin{align}
&P_t^{\mathbf{c}_t}\!\left(
\Pi^\star(\mathbf{c}_t)=a^\star,Y_a=y
\right)
\nonumber\\
&\qquad=
\mathbb{E}_{\boldsymbol{\theta}\sim P_t^{\mathbf{c}_t}(\boldsymbol{\theta})}
\left[
\mathbbm{1}\!\left\{
a^\star
=
\arg\max_{a'\in\mathcal{A}}
\mu_{\boldsymbol{\theta}}(a',\mathbf{c}_t)
\right\}
P\!\left(
Y_a=y
\mid\boldsymbol{\theta},\mathbf{c}_t
\right)
\right],
\end{align}
for every $a^\star\in\mathcal{A}$ and $y\in\{0,1\}$.

Moreover, the expected instantaneous regret associated with action
$a$ at time step $t$ under context $\mathbf{c}_t$ is
\begin{align}
\Delta_t(a\mid\mathbf{c}_t)
&=
\mathbb{E}_{\boldsymbol{\theta}\sim P_t^{\mathbf{c}_t}(\boldsymbol{\theta})}
\Bigg[
\sum_{a^\star\in\mathcal{A}}
\mathbbm{1}\!\left\{
a^\star
=
\arg\max_{a'\in\mathcal{A}}
\mu_{\boldsymbol{\theta}}(a',\mathbf{c}_t)
\right\}
\nonumber\\
&\hspace{3.5cm}\times
P\!\left(
Y_{a^\star}=1
\mid\boldsymbol{\theta},\mathbf{c}_t
\right)
-
P\!\left(
Y_a=1
\mid\boldsymbol{\theta},\mathbf{c}_t
\right)
\Bigg].
\end{align}
\end{lemma}
\vspace{-1em}
When the parameter space is finite and discrete, the quantities
$g_t(a\mid\mathbf{c}_t)$ and $\Delta_t(a\mid\mathbf{c}_t)$ can be
evaluated exactly by summing over all possible values of
$\boldsymbol{\theta}$. In our contextual causal bandit setting,
however, $\boldsymbol{\theta}$ parameterizes the observational
distribution and belongs to a continuous probability simplex.
Consequently, the posterior expectations defining
$g_t(a\mid\mathbf{c}_t)$ and $\Delta_t(a\mid\mathbf{c}_t)$ are
generally analytically intractable. We therefore approximate these
quantities using Monte Carlo samples drawn from the context-conditioned
posterior distribution $P_t^{\mathbf{c}_t}(\boldsymbol{\theta})$.
Let $K=|\mathcal{A}|$. Ignoring the model-dependent
cost of evaluating interventional rewards, this Monte Carlo evaluation
scales as $O(NK)$ per round, while the pairwise IDS optimization in
Algorithm~\ref{alg: IDSaction} considers $O(K^2)$ action pairs.
The following results control the approximation error introduced by
this sampling procedure.

\begin{lemma}[Concentration of the Expected Regret]
\label{lemm_2}
Fix a time step $t$, a realized context
$\mathbf{C}_t=\mathbf{c}_t$, an action $a\in\mathcal{A}$, and
$\delta\in(0,1)$. Let
$\widehat{\Delta}_t(a\mid\mathbf{c}_t)$ be the Monte Carlo estimator of
$\Delta_t(a\mid\mathbf{c}_t)$ computed using $N$ independent samples
from $P_t^{\mathbf{c}_t}(\boldsymbol{\theta})$, and define
\[
r_N^\Delta
:=
\sqrt{
\frac{2}{N}
\log\!\left(\frac{2}{\delta}\right)
}.
\]
Then, with probability at least $1-\delta$ we have
\(
\left|
\widehat{\Delta}_t(a\mid\mathbf{c}_t)
-
\Delta_t(a\mid\mathbf{c}_t)
\right|
\leq
r_N^\Delta.
\)
\end{lemma}

To control the Monte Carlo approximation error in the information gain, we estimate the marginal distributions of $\Pi^\star(\mathbf{c}_t)$ and $Y_a$, together with their joint distribution, using posterior Monte Carlo samples. Using the entropy representation of mutual information and a finite-alphabet entropy-continuity bound, we obtain an explicit confidence interval for $g_t(a\mid\mathbf{c}_t)$ whose width converges to zero as $N$ increases.

\begin{lemma}[Concentration of the Information Gain]
\label{lemm_gt_conc}
Fix $t$, $\mathbf{C}_t=\mathbf{c}_t$, $a\in\mathcal{A}$, and $\delta\in(0,1)$. Let $K:=|\mathcal{A}|\geq2$, $X:=\Pi^\star(\mathbf{c}_t)$, $Z:=Y_a$, and define $\widehat g_t(a\mid\mathbf{c}_t):=H(\widehat P_X)+H(\widehat P_Z)-H(\widehat P_{XZ})$. Let
$\beta_N(\delta):=\sqrt{\frac{1}{2N}\log\!\left(\frac{2(3K+2)}{\delta}\right)}$
and
\[
r_N^{g}:=\phi_K\!\left(K\beta_N(\delta)\right)+\phi_2\!\left(2\beta_N(\delta)\right)+\phi_{2K}\!\left(2K\beta_N(\delta)\right),
\]
where $\phi_m$ is defined in Appendix~\ref{concen_sec}. Then, with probability at least $1-\delta$,
\[
\left|\widehat g_t(a\mid\mathbf{c}_t)-g_t(a\mid\mathbf{c}_t)\right|
\leq r_N^{g},
\]
and hence
$\underline g_t(a\mid\mathbf{c}_t):=\max\{0,\widehat g_t(a\mid\mathbf{c}_t)-r_N^{g}\}\leq g_t(a\mid\mathbf{c}_t)\leq\min\{\log 2,\widehat g_t(a\mid\mathbf{c}_t)+r_N^{g}\}=:\overline g_t(a\mid\mathbf{c}_t)$. Moreover, for fixed $K$ and $\delta$, $r_N^{g}\to0$ as $N\to\infty$.
\end{lemma}

The proof of Lemma~\ref{lemm_gt_conc}, including the concentration of the estimated distributions and the entropy-continuity argument, is provided in Appendix~\ref{concen_sec}.

\begin{theorem}
\label{thm2}
Fix $\delta'\in(0,1)$ and $T\in\mathbb{N}$. For each $t$, let
$\mathcal{E}_t$ denote the event that the confidence bounds in
Lemmas~\ref{lemm_2} and~\ref{lemm_gt_conc} hold for all
$a\in\mathcal{A}$ at round $t$, using
$\delta=\delta'/(2T|\mathcal{A}|)$. Then
$\mathbb{P}(\cap_{t=1}^T\mathcal{E}_t)\geq1-\delta'$, and
\begin{equation}
\mathbb{E}\!\left[
\sum_{t=1}^{T}
\left(
Y_{\Pi^\star(\mathbf{C}_t),t}
-
Y_{A_t,t}
\right)
\right]
\leq
\sqrt{
H(\Pi^\star)
\left(
\frac{T|\mathcal{A}|}{2}
+
2\sum_{t=1}^{T}
\mathbb{E}\!\left[
\gamma_t\mathbbm{1}_{\mathcal{E}_t}
\right]
\right)
}
+
\delta',
\label{eq:ids-regret-bound}
\end{equation}
where, with
$S_{|\mathcal{A}|}:=
\{\pi\in\mathbb{R}^{|\mathcal{A}|}_{+}:\sum_a\pi(a)=1\}$,
\begin{equation}
\label{eq:gamma-t}
\gamma_t:=
\sup_{\substack{
\pi\in S_{|\mathcal{A}|},\;
\boldsymbol{\Delta}_t^{\,1},\boldsymbol{\Delta}_t^{\,2}
\in\mathcal{C}_t^\Delta(\mathbf{C}_t),\\
\boldsymbol{g}_t^{\,1},\boldsymbol{g}_t^{\,2}
\in\mathcal{C}_t^g(\mathbf{C}_t)
}}
\left|
\frac{
(\pi^\top\boldsymbol{\Delta}_t^{\,1})^2
}{
\pi^\top\boldsymbol{g}_t^{\,1}
}
-
\frac{
(\pi^\top\boldsymbol{\Delta}_t^{\,2})^2
}{
\pi^\top\boldsymbol{g}_t^{\,2}
}
\right|
\in[0,+\infty].
\end{equation}
Here $\mathcal{C}_t^\Delta(\mathbf{C}_t)$ and
$\mathcal{C}_t^g(\mathbf{C}_t)$ are the componentwise confidence sets
from Lemmas~\ref{lemm_2} and~\ref{lemm_gt_conc}, respectively. We use
$0/0:=0$ and $x/0:=+\infty$ for $x>0$. If an admissible
information-gain denominator vanishes while the numerator
is positive, then $\gamma_t=+\infty$, in which case the finite-sample
bound is valid but vacuous.
\end{theorem}

Theorem~\ref{thm2} separates the usual information-theoretic IDS term
from the effect of Monte Carlo approximation. The confidence radius
$r_N^\Delta$ in Lemma~\ref{lemm_2} converges to zero at the usual Monte
Carlo rate, while the information-gain radius $r_N^g$ in
Lemma~\ref{lemm_gt_conc}, obtained through finite-alphabet entropy
continuity, also converges to zero as $N\to\infty$. Thus, for fixed $T$,
$|\mathcal{A}|$, and $\delta'$, the diameters of both componentwise
confidence sets vanish as $N$ increases. The choice
$\delta=\delta'/(2T|\mathcal{A}|)$ ensures simultaneous coverage of both
confidence bounds over all rounds and actions with probability at least
$1-\delta'$.

The behavior of $\gamma_t$ additionally depends on the information-gain
denominator. Suppose, for a fixed round $t$, history, and realized
context, that all denominators appearing in the supremum defining
$\gamma_t$ satisfy $\pi^\top\boldsymbol{g}\geq\rho_t>0$. Let
$w_{\Delta,t}$ and $w_{g,t}$ denote the maximum coordinate-wise widths
of $\mathcal{C}^{\Delta}_t(\mathbf{C}_t)$ and
$\mathcal{C}^{g}_t(\mathbf{C}_t)$, respectively. Since these intervals
have radii $r_N^\Delta$ and $r_N^g$, with clipping only reducing their
widths, $w_{\Delta,t}\leq2r_N^\Delta$ and
$w_{g,t}\leq2r_N^g$. For any two admissible pairs
$(\boldsymbol{\Delta}^1,\boldsymbol{g}^1)$ and
$(\boldsymbol{\Delta}^2,\boldsymbol{g}^2)$, write
$x_i=\pi^\top\boldsymbol{\Delta}^i$ and
$y_i=\pi^\top\boldsymbol{g}^i$. Since rewards lie in $[0,1]$,
$0\leq x_i\leq1$ and $y_i\geq\rho_t$. By adding and subtracting
$x_2^2/y_1$,
$\left|x_1^2/y_1-x_2^2/y_2\right|
\leq |x_1^2-x_2^2|/y_1+
x_2^2|y_1-y_2|/(y_1y_2)$.
Moreover, $|x_1^2-x_2^2|\leq2|x_1-x_2|$,
$|x_1-x_2|\leq w_{\Delta,t}$, and
$|y_1-y_2|\leq w_{g,t}$. Therefore,
\begin{equation}
\gamma_t
\leq
\frac{2w_{\Delta,t}}{\rho_t}
+
\frac{w_{g,t}}{\rho_t^2}
\leq
\frac{4r_N^\Delta}{\rho_t}
+
\frac{2r_N^g}{\rho_t^2}.
\label{eq:gamma-width-bound}
\end{equation}

This also gives a concrete interpretation of the required Monte Carlo
budget $N$. Let $\boldsymbol{g}_t^\star$ denote the true
information-gain vector and suppose
$\pi^\top\boldsymbol{g}_t^\star\geq\underline{\rho}_t>0$ for every
relevant $\pi$. On the confidence event, any admissible
$\boldsymbol{g}$ differs coordinate-wise from
$\boldsymbol{g}_t^\star$ by at most $2r_N^g$, and hence
$|\pi^\top\boldsymbol{g}-\pi^\top\boldsymbol{g}_t^\star|
\leq2r_N^g$. Therefore, if
$r_N^g\leq\underline{\rho}_t/4$, every admissible denominator is at
least $\underline{\rho}_t-2r_N^g\geq\underline{\rho}_t/2$.
Equation~\eqref{eq:gamma-width-bound} then gives
$\gamma_t\leq8r_N^\Delta/\underline{\rho}_t+
8r_N^g/\underline{\rho}_t^2$. Thus, for any target
$\varepsilon>0$, it is sufficient that
$r_N^\Delta\leq\varepsilon\underline{\rho}_t/16$ and
$r_N^g\leq\min\{\underline{\rho}_t/4,
\varepsilon\underline{\rho}_t^2/16\}$.

For fixed $T$, $|\mathcal{A}|$, and overall confidence level
$\delta'$, the choice $\delta=\delta'/(2T|\mathcal{A}|)$ is fixed.
Consequently, $r_N^\Delta=O(N^{-1/2})$ and, by finite-alphabet
entropy continuity, $r_N^g=O((\log N)/\sqrt{N})$ (see
Appendix~\ref{app:mc_rates}). Hence, for fixed
$\varepsilon>0$ and $\underline{\rho}_t>0$, these sufficient
conditions hold for all sufficiently large $N$, yielding a finite
Monte Carlo budget. More generally, if the relevant information-gain
denominators remain bounded away from zero as $N$ increases, then
$r_N^\Delta\to0$ and $r_N^g\to0$ imply $\gamma_t\to0$. Without such
separation, this conclusion is not guaranteed: if a denominator
approaches zero while the corresponding regret numerator remains
positive, $\gamma_t$ can become arbitrarily large or equal $+\infty$.
Under the conventions $0/0:=0$ and $x/0:=+\infty$ for $x>0$, this
possibility is already captured by the definition of $\gamma_t$, and
the finite-sample bound may therefore be vacuous. Thus, $\gamma_t$
captures both Monte Carlo approximation error and its amplification
when information gain is small.

\begin{corollary}[IDS with Oracle Access]
\label{cor:oracle}
Suppose that $\Delta_t(a\mid\mathbf{c}_t)$ and $g_t(a\mid\mathbf{c}_t)$ can be computed exactly at every round and IDS minimizes the exact information ratio. Then, for any $T\in\mathbb{N}$, the mean cumulative regret satisfies
\[
\mathbb{E}\!\left[
\sum_{t=1}^T
\left(
Y_{\Pi^\star(\mathbf{C}_t),t}
-
Y_{A_t,t}
\right)
\right]
\leq
\sqrt{
\frac{|\mathcal{A}|}{2}
\,T\,H(\Pi^\star)
}.
\]
\end{corollary}

This result follows directly from the standard IDS analysis when the information ratio is computed exactly. To illustrate how causal structure can affect the prior distribution over optimal interventions and thereby tighten the entropy-dependent regret bound, consider the causal graph in Figure~\ref{exp_graphs}(a), consisting of the binary nodes $X$, $Y$, and $Z$, with edges $X\to Y$, $Z\to X$, and $Z\to Y$, where $Y$ is the reward node and $Z$ is non-manipulable. There is no context variable in this example. Because $Z$ cannot be manipulated, the feasible intervention set is $\mathcal{A}=\{do(),do(X=1),do(X=0)\}$. Unlike an unstructured bandit model in which the arms are assigned unrelated reward parameters, the expected rewards of these interventions are jointly determined by the same causal mechanisms and shared conditional probability tables. Thus, the causal model induces a joint prior over the intervention rewards and, consequently, a prior distribution over the optimal intervention $\Pi^\star$. Assuming no prior observations, we place independent $\operatorname{Beta}(1,1)$ priors on the Bernoulli parameters in each row of the conditional probability tables and estimate $P(\Pi^\star=a)$ using $N=50{,}000$ Monte Carlo samples. We obtain $P(\Pi^\star=do())\approx0.08$, $P(\Pi^\star=do(X=1))\approx0.46$, and $P(\Pi^\star=do(X=0))\approx0.46$. By Hoeffding's inequality and a union bound over the three interventions, these estimates have a simultaneous error bound of approximately $\pm0.0093$ at overall confidence level $0.999$. The corresponding estimated entropy is
$H(\Pi^\star)=-\sum_{a\in\mathcal{A}}P(\Pi^\star=a)\log P(\Pi^\star=a)\approx0.9165$,
which is smaller than the worst-case value $\log|\mathcal{A}|=\log3\approx1.0986$. Substituting this estimated entropy into the oracle IDS bound gives approximately
$\mathbb{E}[\mathrm{Reg}_T]\leq1.1725\sqrt{T}$,
whereas replacing $H(\Pi^\star)$ by its worst-case upper bound $\log3$ gives approximately
$1.2837\sqrt{T}$. Thus, for this prior, the leading coefficient obtained from the estimated entropy is approximately $8.7\%$ smaller than the coefficient obtained from the worst-case entropy bound. Moreover, because the interventions share the same causal mechanisms, observations collected under one intervention can update beliefs about the rewards and information gains of other interventions, although this additional cross-intervention information sharing is not explicitly captured by the worst-case factor $|\mathcal{A}|/2$.

\section{Experiments}

We empirically compare the proposed algorithms with causal and non-causal
baselines on three structured synthetic tasks of increasing complexity. The
corresponding causal graphs are shown in Figure~\ref{exp_graphs}, where red
nodes denote non-manipulable variables, $C$ denotes a context variable, and
$Y$ denotes the reward variable. In all experiments, the causal graph and the
set of feasible interventions are known to the learner, whereas the numerical
values of the conditional probability tables (CPTs) governing the
data-generating process are unknown. We use one fixed SCM for each task,
shared by all algorithms and all independent trials. The complete CPT
specifications of the three SCMs are provided in
Appendix~\ref{app:experimental_scms}.

\begin{figure*}[t!]
\begin{subfigure}[b]{3.8cm}
\centering
    \begin{tikzpicture}
       \node (1) {$X$};
    \node (2) [right =of 1] {$Y$};
    
    \node[red] (4) [above right =of 1,xshift=-0.63cm,yshift=-0.1cm]  {$Z$};

  \path (1) edge (2);
  \path (4) edge  (1);
   \path (4) edge  (2);
    \end{tikzpicture}
    \subcaption{Task $1$}
\end{subfigure}
\hfill
\begin{subfigure}[b]{5.2cm}
    \centering
    \begin{tikzpicture}
       \node (1) {$X_2$};
    \node (2) [right =of 1] {$Y$};
    \node (3) [left =of 1]{$X_1$};
    \node[orange] (5) [above right =of 3,xshift=-0.63cm,yshift=-0.1cm]  {$C$};
    \node[red] (4) [above right =of 1,xshift=-0.63cm,yshift=-0.1cm]  {$Z$};
    
  \path (5) edge (1); 
  \path (5) edge (3); 
  \path (1) edge (2);
  \path (4) edge  (1);
   \path (4) edge  (2);
   \path (3) edge  (1);
    \end{tikzpicture}
    \subcaption{Task $2$}
\end{subfigure} 
\hfill
\begin{subfigure}[b]{5.2cm}
\centering
    \begin{tikzpicture}
        \node (1) {$X_1$};
    \node (2) [right =of 1, xshift = -0.11cm] {$Y$};
    \node (3) [right =of 2, xshift = -0.11cm] {$X_2$};
     \node (5) [right =of 3, xshift = -0.11cm] {$X_3$};
    
    \node[red] (4) [above right =of 1,xshift=-0.63cm,yshift=-0.1cm]  {$Z_1$};
    \node[red] (6) [above  =of 3,yshift = 0.1cm]  {$Z_2$};

        \node[orange] (7) [above  =of 3,yshift = -0.5cm]  {$C$};
  \path (7) edge (2);
   \path (7) edge (5);
  \path (6) edge (2);
  \path (6) edge (5);
  \path (1) edge (2);
  \path (4) edge  (1);
  \path (5) edge  (3);
   \path (4) edge  (2);
   \path (3) edge  (2); 
    \end{tikzpicture}
   \subcaption{Task $3$}
\end{subfigure}
\caption{Causal graphs used in the experiments}
\label{exp_graphs}
\vspace{-0.3em}
\end{figure*}
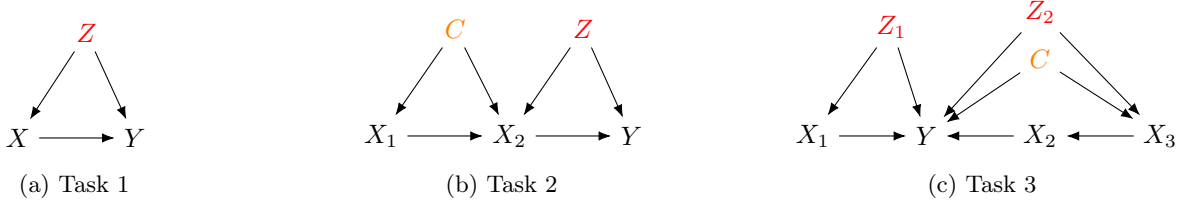

\begin{figure*}[t!]
\centering
\begin{subfigure}[b]{5.2cm}
    \includegraphics[height = 3.5cm, width=5.2cm]{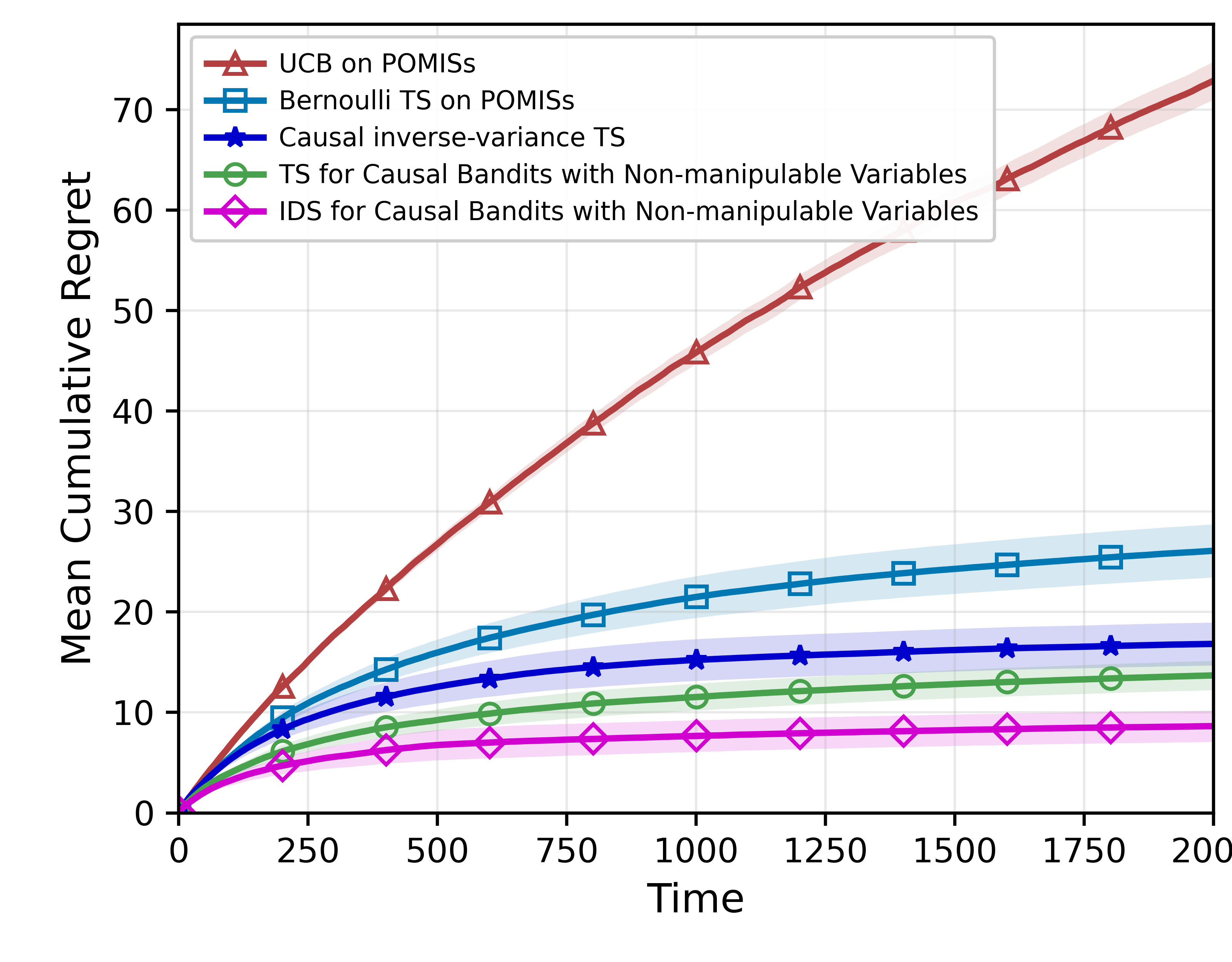}
    \subcaption{Task $1$}
    \end{subfigure}
\enspace
\begin{subfigure}[b]{5.2cm}
    \includegraphics[height = 3.5cm,width=5.2cm]{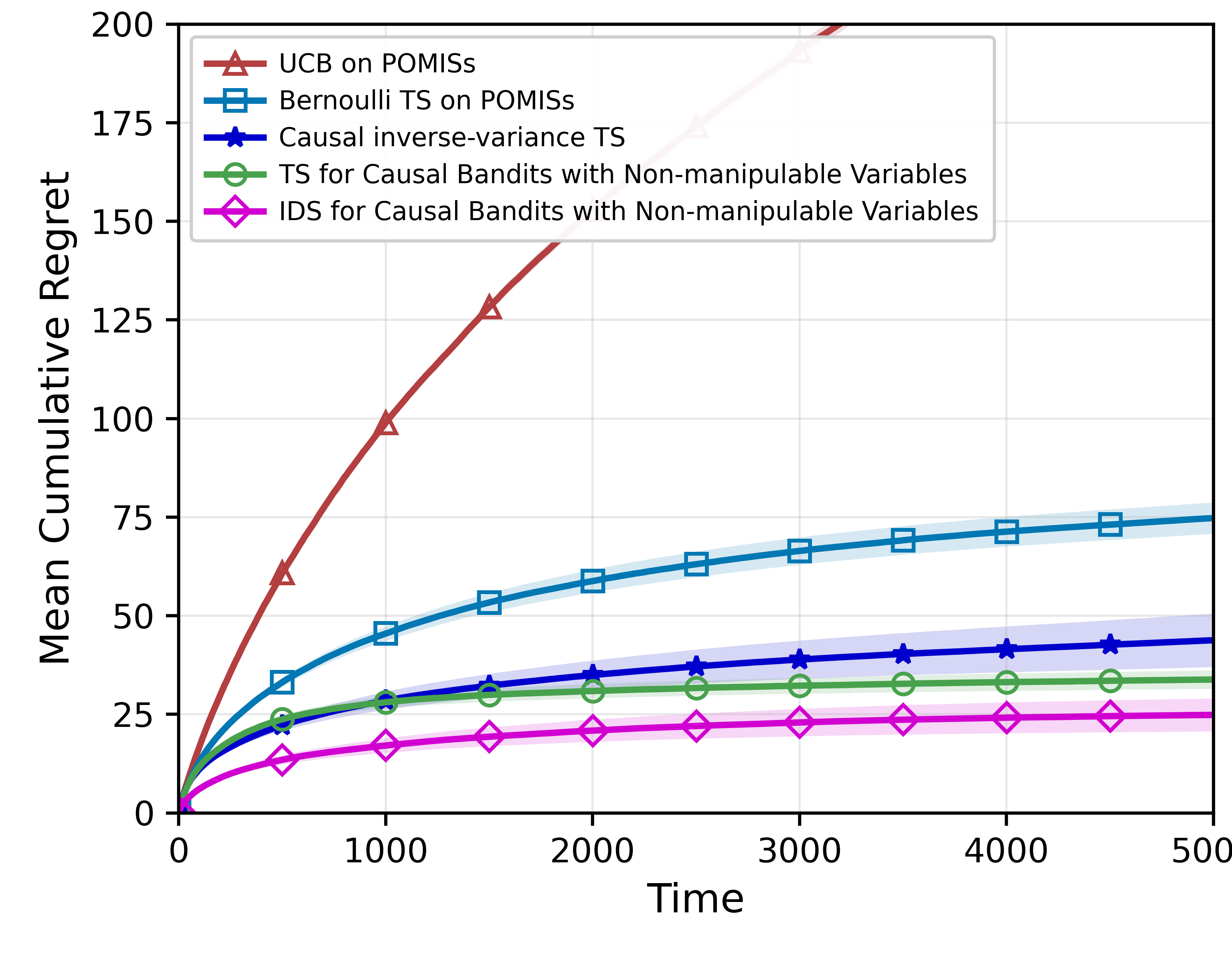}
   \subcaption{Task $2$}
    \end{subfigure}
\enspace
\begin{subfigure}[b]{5.2cm}
    \includegraphics[height = 3.5cm,width=5.2cm]{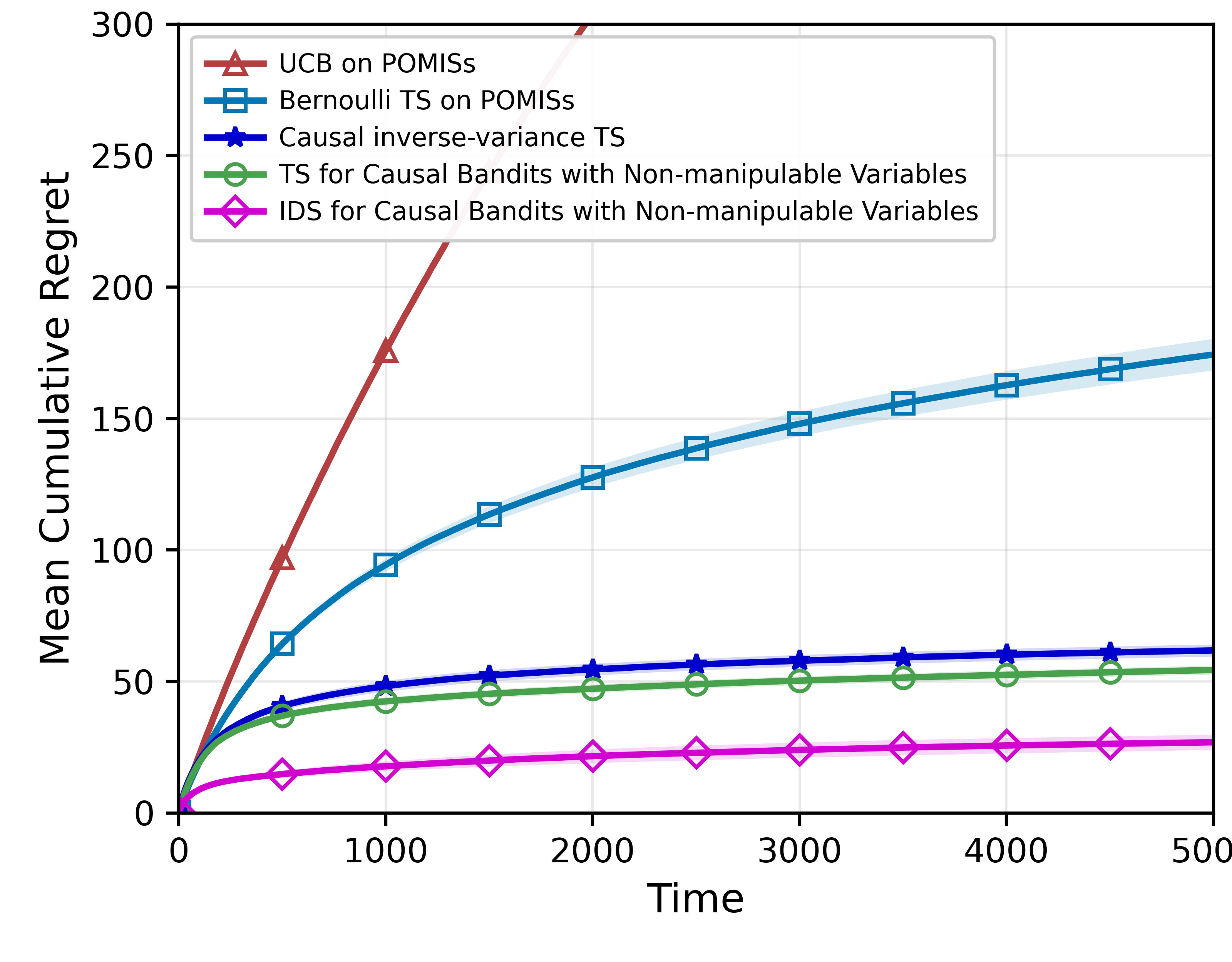}
    \subcaption{Task $3$}
\end{subfigure}  
\vspace{-0.3em}
\caption{Regret versus time for the three tasks with the corresponding causal graphs shown in Figure~\ref{exp_graphs}}

\label{exp_res_1}
\end{figure*}

For \textbf{Task 1} (Figure~\ref{exp_graphs}(a)), the set of POMISs after the
projection step is $\{\emptyset,\{X\}\}$, resulting in the possibly optimal
arms $do()$, $do(X=0)$, and $do(X=1)$. For \textbf{Task 2}
(Figure~\ref{exp_graphs}(b)), the set of POMISs is
$\{\{X_1\},\{X_2\}\}$, yielding the possibly optimal arms
$do(X_1=0)$, $do(X_1=1)$, $do(X_2=0)$, and $do(X_2=1)$.
Here the context variable takes three values, $C\in\{0,1,2\}$, and the
identity of the optimal intervention depends on the realized context. For
\textbf{Task 3} (Figure~\ref{exp_graphs}(c)), the set of POMISs is
$\{\emptyset,\{X_1\},\{X_2\},\{X_1,X_2\}\}$, resulting in a total of nine
possibly optimal arms. We again use $C\in\{0,1,2\}$ with
context-dependent optimal interventions. Tasks 2 and 3 therefore evaluate
the algorithms in settings where the learner must adapt its action to the
observed context while exploiting causal mechanisms shared across
interventions. We use horizon $T=2000$ for Task 1 and $T=5000$ for Tasks 2 and 3.

We consider three baselines. The first two are non-causal methods: vanilla
UCB and Bernoulli Thompson Sampling (TS), applied directly to the set of
possibly optimal interventions induced by the POMISs. In the contextual
tasks, they maintain context-specific arm-level reward statistics without
exploiting the factorization of the causal model.  We additionally consider a
causal inverse-variance Thompson-sampling baseline inspired by
Lee and Bareinboim~\cite{lee2019structural}. It maintains a direct arm-level
reward posterior, or context--arm posterior in contextual tasks, initialized
with a $\operatorname{Beta}(1,1)$ prior and updated from the corresponding
rewards. It also maintains a shared-CPT posterior whose Bernoulli rows are
initialized with independent $\operatorname{Beta}(0.5,0.5)$ priors and updated
from observed non-intervened variables. At each round, $50$ shared-CPT
posterior samples provide an auxiliary causal reward estimate and variance.
The direct and auxiliary estimates are combined by inverse-variance weighting,
and the resulting mean and variance are moment-matched to a Beta distribution
from which a Thompson sample is drawn. In contextual tasks, all reward
quantities are conditioned on the observed context. The proposed causal TS and IDS algorithms instead maintain a Bayesian posterior over the shared CPT parameters of the causal model. Causal TS samples one SCM from the posterior each round and selects the intervention with the highest expected reward under the realized context. IDS uses posterior SCM samples to estimate each intervention's instantaneous regret and information gain about the context-dependent optimal action, then minimizes the resulting information ratio. After each interaction, observed non-intervened variables update their CPT rows, allowing data from one intervention to refine posterior reward estimates for several others.

For each task, all algorithms are evaluated on exactly the same fixed SCM and
therefore face the same underlying data-generating process. For Tasks 2 and
3, a new context is sampled independently from the task-specific context
distribution at every round and is revealed before action selection. We
evaluate performance using cumulative regret, with contextual
regret used in Tasks 2 and 3, where the instantaneous regret is the
difference between the expected reward of the optimal intervention for the
realized context and that of the selected intervention. For every algorithm
and every task, we run $500$ independent trials and report in
Figure~\ref{exp_res_1} the mean cumulative regret across these trials.
The shaded region around each curve is the pointwise $95\%$ confidence
interval for the mean, computed at each round $t$ as
$\bar R_t \pm 1.96\,s_t/\sqrt{500}$, where $\bar R_t$ and $s_t$ denote the
sample mean and sample standard deviation of cumulative regret across
the $500$ independent trials at round $t$. Thus, the shaded region quantifies
sampling uncertainty in the estimated mean regret at each individual round,
rather than the variability of individual trial trajectories, and the
intervals are pointwise in time rather than a simultaneous confidence band
for the complete trajectory.

For IDS, we use $N=1000$ posterior samples per round in \emph{all three
tasks}. In the experimental implementation, we use the plug-in Monte Carlo
estimates $\widehat{\Delta}_t(a\mid\mathbf{c}_t)$ and
$\widehat{g}_t(a\mid\mathbf{c}_t)$ directly when optimizing the information
ratio, where $\widehat{g}_t(a\mid\mathbf{c}_t)$ estimates the information
provided by the reward observation about the context-dependent optimal
action. Algorithm~\ref{alg: ids} is stated in confidence-set form for the
finite-sample analysis; the confidence intervals in
Lemmas~\ref{lemm_2} and~\ref{lemm_gt_conc} are not used for experimental
action selection. These simultaneous high-probability bounds are worst-case
and can be conservative at practical Monte Carlo sample sizes. For example,
at $N=1000$ and $\delta=0.01$, the information-gain radius $r_N^g$ is
approximately $1.39$, $1.81$, and $3.82$ for
$|\mathcal{A}|=3,4,$ and $9$, respectively, whereas
$g_t(a\mid\mathbf{c}_t)\leq\log 2\approx0.69$ because the reward is binary.
Thus, the corresponding worst-case confidence intervals may contain zero,
causing the $\gamma_t$-dependent finite-sample bound to become vacuous even
when the plug-in Monte Carlo estimates are sufficiently accurate for
practical action selection. We therefore use the common value $N=1000$
across all three tasks as a practical choice that yields stable Monte Carlo
estimates while avoiding unnecessary computational cost. 
Appendix~\ref{app_sens} further examines the sensitivity of IDS to the
Monte Carlo budget on Task~3, using $N\in\{250,1000,4000\}$.
Theorem~\ref{thm2} should accordingly be interpreted as a conservative
worst-case finite-sample guarantee whose informativeness depends on both the
Monte Carlo confidence radii and sufficient separation of the
information-gain denominator.

Overall, the proposed causal algorithms achieve lower cumulative regret on
these structured instances by exploiting mechanisms shared across
interventions. IDS further improves over causal TS by balancing expected
instantaneous regret against information gain, favoring interventions
informative about the context-dependent optimal action. These experiments
demonstrate the instance-specific benefits of mechanism sharing and
information-directed exploration across different intervention spaces and
context-dependent policies.

\section{Conclusion}

We study contextual causal bandits with non-manipulable variables and introduce Bayesian Thompson Sampling (TS) and Information-Directed Sampling (IDS) algorithms for this setting. By maintaining a posterior over the conditional probability tables of the causal model, the proposed methods exploit information shared across interventions to improve learning efficiency. We establish an entropy-dependent sublinear Bayesian regret bound for Thompson Sampling. For IDS, we derive a regret guarantee that explicitly captures the additional error introduced by Monte Carlo approximation, while the oracle version recovers the standard sublinear IDS bound. Empirical results show that the proposed algorithms outperform existing causal and non-causal baselines on the considered structured instances. Our regret guarantees are worst-case, while sharper mechanism-dependent bounds remain an important direction for future work. Overall, our framework provides a principled approach for sequential decision-making in causal systems containing variables that influence the reward but cannot be directly manipulated.

\section{Acknowledgements}
This research has been supported in part by NSF CAREER 2617987, IIS 2617859, Adobe Research and Intuit.

\clearpage

\bibliographystyle{plainnat}
\bibliography{main}  

\newpage
\appendix

\section{Information-Directed Sampling for Causal Bandits (Supplementary Material)}

\subsection{Optimizing the Contextual Information Ratio}

If $\vec{g}_t(\mathbf{c}_t)=\vec{0}$, then none of the available actions provides information about the optimal action under context $\mathbf{c}_t$. In this case, we select an action that minimizes the expected instantaneous regret. Otherwise, at time $t$ we select an action by solving
\begin{equation}
\min_{\pi \in S_{|\mathcal{A}|}}
\frac{\left(\pi^\top \vec{\Delta}_t(\mathbf{c}_t)\right)^2}
{\pi^\top \vec{g}_t(\mathbf{c}_t)},
\end{equation}
where $S_{|\mathcal{A}|}=\{\pi \in \mathbb{R}_+^{|\mathcal{A}|} : \sum_k \pi_k=1\}$ is the probability simplex. We use the conventions $0/0:=0$ and $x/0:=+\infty$ for $x>0$.

Russo and Van Roy~\cite{russo2014learning} show that this optimization problem is convex and admits an optimal solution with at most two nonzero components. Therefore, although IDS is a randomized policy, it suffices to randomize over at most two actions at each round.

\begin{algorithm}[h]
\small
\DontPrintSemicolon
\SetKwFunction{FMain}{IDSAction}
\SetKwProg{Fn}{Function}{:}{}
\Fn{\FMain{$\mathcal{A},\vec{\Delta}_t(\mathbf{c}_t),\vec{g}_t(\mathbf{c}_t)$}}{

\If{$\vec{g}_t(\mathbf{c}_t)=\vec{0}$}{
\textbf{return} an action $a^*\in\arg\min_{a\in\mathcal{A}}\Delta_t(a\mid\mathbf{c}_t)$.
}

For each pair $a,a' \in \mathcal{A}$ such that $a\neq a'$ compute
\[
q_{a,a'} \leftarrow
\arg\min_{q\in[0,1]}
\frac{
\left(q\,\Delta_t(a\mid\mathbf{c}_t)
+(1-q)\,\Delta_t(a'\mid\mathbf{c}_t)\right)^2}
{q\,g_t(a\mid\mathbf{c}_t)
+(1-q)\,g_t(a'\mid\mathbf{c}_t)}.
\]

Select
\[
(a^*,a^{**})
\leftarrow
\arg\min_{a\neq a'}
\frac{
\left(q_{a,a'}\,\Delta_t(a\mid\mathbf{c}_t)
+(1-q_{a,a'})\,\Delta_t(a'\mid\mathbf{c}_t)\right)^2}
{q_{a,a'}\,g_t(a\mid\mathbf{c}_t)
+(1-q_{a,a'})\,g_t(a'\mid\mathbf{c}_t)}.
\]

\textbf{return} $a^*$ with probability $q_{a^*,a^{**}}$
and $a^{**}$ with probability $1-q_{a^*,a^{**}}$.
}
\textbf{End Function}
\caption{Function to optimize the contextual information ratio}
\label{alg: IDSaction}
\end{algorithm}

\subsection{Proof of Theorem~\ref{thm1}}

\begin{equation}
\mathbb{E}[\mathrm{Reg}_T]
=
\mathbb{E}\!\left[\sum_{t=1}^T \left( Y_{\Pi^\star(\mathbf{C}_t),t} - Y_{A_t,t} \right)\right]
=
\mathbb{E}\!\left[\sum_{t=1}^T \mathbb{E}_t\!\left[ Y_{\Pi^\star(\mathbf{C}_t),t} - Y_{A_t,t} \mid \mathbf{C}_t\right]\right].
\end{equation}

We introduce a nonnegative constant $\Gamma_t$ such that the per-round regret satisfies
\begin{equation}
\mathbb{E}_t\!\left[
Y_{\Pi^\star(\mathbf{C}_t),t} - Y_{A_t,t}
\mid \mathbf{C}_t\right]
\leq
\Gamma_t 
\sqrt{
I_t\!\left(\Pi^\star(\mathbf{C}_t) ; (A_t,Y_{A_t,t}) \mid \mathbf{C}_t \right)
}.
\end{equation}

Then,
\begin{align*}
\mathbb{E}[\mathrm{Reg}_T]
&\leq
\mathbb{E}\!\left[
\sum_{t=1}^T 
\Gamma_t 
\sqrt{
I_t\!\left(\Pi^\star(\mathbf{C}_t) ; (A_t,Y_{A_t,t}) \mid \mathbf{C}_t \right)
}
\right] \\
&\leq
\Gamma \,
\mathbb{E}\!\left[
\sum_{t=1}^T 
\sqrt{
I_t\!\left(\Pi^\star(\mathbf{C}_t) ; (A_t,Y_{A_t,t}) \mid \mathbf{C}_t \right)
}
\right] \\
&\leq
\Gamma 
\sqrt{
T \,
\mathbb{E}\!\left[
\sum_{t=1}^T 
I_t\!\left(\Pi^\star(\mathbf{C}_t) ; (A_t,Y_{A_t,t}) \mid \mathbf{C}_t \right)
\right]
}.
\end{align*}

The second inequality follows from the assumption that $\Gamma \geq \Gamma_t$ for all $t \in \mathbb{N}$, and the third inequality follows from the Cauchy--Schwarz and Jensen inequalities.

Since $\Pi^\star(\mathbf{C}_t)$ is a deterministic function of the full optimal-action mapping $\Pi^\star$, the data-processing inequality gives
\begin{equation}
I_t\!\left(\Pi^\star(\mathbf{C}_t) ; (A_t,Y_{A_t,t}) \mid \mathbf{C}_t \right)
\leq
I_t\!\left(\Pi^\star ; (A_t,Y_{A_t,t}) \mid \mathbf{C}_t \right).
\end{equation}

Let $\mathcal{O}_t$ denote the complete observation collected at time $t$, including $\mathbf{C}_t$, $A_t$, $Y_{A_t,t}$, and all observed non-intervened variables, so that $\mathcal{F}_{t+1}=\sigma(\mathcal{F}_t,\mathcal{O}_t)$. By the chain rule for mutual information,
\begin{equation}
I_t\!\left(\Pi^\star ; (A_t,Y_{A_t,t}) \mid \mathbf{C}_t \right)
\leq
I_t\!\left(\Pi^\star ; \mathcal{O}_t\right).
\end{equation}

By the definition of conditional mutual information,
\begin{equation*}
I_t\!\left(\Pi^\star ; \mathcal{O}_t\right)
=
H\!\left(\Pi^\star \mid \mathcal{F}_t\right)
-
\mathbb{E}\!\left[
H\!\left(\Pi^\star \mid \mathcal{F}_{t+1}\right)
\mid \mathcal{F}_t
\right].
\end{equation*}

Summing over $t$ yields a telescoping series:
\begin{equation}
\mathbb{E}\!\left[
\sum_{t=1}^T 
I_t\!\left(\Pi^\star(\mathbf{C}_t) ; (A_t,Y_{A_t,t}) \mid \mathbf{C}_t \right)
\right]
\leq
H(\Pi^\star)
-
\mathbb{E}\!\left[
H\!\left(\Pi^\star \mid \mathcal{F}_{T+1}\right)
\right]
\leq
H(\Pi^\star).
\end{equation}

Substituting this bound gives
\begin{equation}
\mathbb{E}[\mathrm{Reg}_T]
\leq
\Gamma \sqrt{ H(\Pi^\star) \, T }.
\end{equation}

This completes the proof of Theorem~\ref{thm1}.

\subsection{Proof of Corollary \ref{crr_reg}}

We fix a realization of $\mathcal{F}_t$ and condition on the realized context $\mathbf{C}_t=\mathbf{c}$.
Write $p_a := P_t(\Pi^\star(\mathbf{c})=a\mid \mathbf{C}_t=\mathbf{c})$.
By the TS posterior-matching property, $P_t(A_t=a\mid \mathbf{C}_t=\mathbf{c})=p_a$. Moreover, conditional on $\mathcal{F}_t$ and $\mathbf{C}_t=\mathbf{c}$, the action $A_t$ and the optimal action $\Pi^\star(\mathbf{c})$ are independent because $A_t$ is generated using an independent posterior sample. Also let $K = |\mathcal{A}|$ be the total number of actions.

Define the posterior mean reward of action $a$ at context $\mathbf{c}$:
\[
\bar\mu_a := E_t[Y_{a,t}\mid \mathbf{C}_t=\mathbf{c}].
\]
Also define the posterior mean reward under the optimal action for that context:
\[
\bar\mu^\star := E_t[Y_{\Pi^\star(\mathbf{c}),t}\mid \mathbf{C}_t=\mathbf{c}]
= \sum_{a\in\mathcal{A}} p_a\, E_t[Y_{a,t}\mid \Pi^\star(\mathbf{c})=a, \mathbf{C}_t=\mathbf{c}].
\]

Then the conditional instantaneous regret is
\begin{align}
\Delta_t(\mathbf{c})
&:= E_t\!\left[Y_{\Pi^\star(\mathbf{c}),t} - Y_{A_t,t}\mid \mathbf{C}_t=\mathbf{c}\right]\nonumber\\
&= \bar\mu^\star - \sum_{a\in\mathcal{A}} P_t(A_t=a\mid \mathbf{C}_t=\mathbf{c})\,\bar\mu_a \nonumber\\
&= \sum_{a\in\mathcal{A}} p_a\left(E_t[Y_{a,t}\mid \Pi^\star(\mathbf{c})=a, \mathbf{C}_t=\mathbf{c}] - \bar\mu_a\right).
\label{eq:delta-expand}
\end{align}

Applying the Cauchy--Schwarz inequality gives
\begin{align}
\Delta_t(\mathbf{c})^2
&\le K \sum_{a\in\mathcal{A}} p_a^2\left(E_t[Y_{a,t}\mid \Pi^\star(\mathbf{c})=a, \mathbf{C}_t=\mathbf{c}] - \bar\mu_a\right)^2.
\label{eq:cs1}
\end{align}

Because rewards are Bernoulli, for each $a$ the two distributions
\[
P_t(Y_{a,t}\in\cdot \mid \Pi^\star(\mathbf{c})=a, \mathbf{C}_t=\mathbf{c})
\quad\text{and}\quad
P_t(Y_{a,t}\in\cdot \mid \mathbf{C}_t=\mathbf{c})
\]
are distributions on $\{0,1\}$, and their total variation distance equals the absolute difference of their means:
\[
TV\!\left(P_t(Y_{a,t}\mid \Pi^\star(\mathbf{c})=a, \mathbf{C}_t=\mathbf{c}),\,P_t(Y_{a,t}\mid \mathbf{C}_t=\mathbf{c})\right)
=
\left|E_t[Y_{a,t}\mid \Pi^\star(\mathbf{c})=a, \mathbf{C}_t=\mathbf{c}]-\bar\mu_a\right|.
\]

Applying Pinsker's inequality for each $a$ gives
\begin{align}
\left(E_t[Y_{a,t}\mid \Pi^\star(\mathbf{c})=a, \mathbf{C}_t=\mathbf{c}]-\bar\mu_a\right)^2
&\le \frac{1}{2}KL\!\left(P_t(Y_{a,t}\mid \Pi^\star(\mathbf{c})=a, \mathbf{C}_t=\mathbf{c})\ \|\ P_t(Y_{a,t}\mid \mathbf{C}_t=\mathbf{c})\right).
\label{eq:pinsker}
\end{align}

Plugging \eqref{eq:pinsker} into \eqref{eq:cs1} yields
\begin{align}
\Delta_t(\mathbf{c})^2
&\le \frac{K}{2}\sum_{a\in\mathcal{A}} p_a^2\,
KL\!\left(P_t(Y_{a,t}\mid \Pi^\star(\mathbf{c})=a, \mathbf{C}_t=\mathbf{c})\ \|\ P_t(Y_{a,t}\mid \mathbf{C}_t=\mathbf{c})\right).
\label{eq:cs2}
\end{align}

We now relate the right-hand side to the information obtained under Thompson Sampling. Since $A_t$ and $\Pi^\star(\mathbf{c})$ are conditionally independent given $\mathcal{F}_t$ and $\mathbf{C}_t=\mathbf{c}$, we have
\begin{align}
&I_t\!\left(\Pi^\star(\mathbf{c});(A_t,Y_{A_t,t})\mid \mathbf{C}_t=\mathbf{c}\right)\nonumber\\
&=
\sum_{a\in\mathcal{A}}p_a\,
I_t\!\left(\Pi^\star(\mathbf{c});Y_{a,t}\mid \mathbf{C}_t=\mathbf{c}\right)\nonumber\\
&=
\sum_{a\in\mathcal{A}}\sum_{a'\in\mathcal{A}}p_a p_{a'}\,
KL\!\left(
P_t(Y_{a,t}\mid \Pi^\star(\mathbf{c})=a',\mathbf{C}_t=\mathbf{c})
\ \|\ 
P_t(Y_{a,t}\mid \mathbf{C}_t=\mathbf{c})
\right)\nonumber\\
&\geq
\sum_{a\in\mathcal{A}}p_a^2\,
KL\!\left(
P_t(Y_{a,t}\mid \Pi^\star(\mathbf{c})=a,\mathbf{C}_t=\mathbf{c})
\ \|\ 
P_t(Y_{a,t}\mid \mathbf{C}_t=\mathbf{c})
\right).
\label{eq:kl-sum}
\end{align}

Combining \eqref{eq:cs2} and \eqref{eq:kl-sum} gives
\[
\Delta_t(\mathbf{c})^2
\leq
\frac{K}{2}
I_t\!\left(\Pi^\star(\mathbf{c});(A_t,Y_{A_t,t})\mid \mathbf{C}_t=\mathbf{c}\right).
\]

Therefore, for every realized context $\mathbf{c}$,
\[
\Gamma_t^2\leq\frac{K}{2},
\]
or equivalently,
\[
\Gamma_t\leq\sqrt{\frac{K}{2}}
=
\sqrt{\frac{|\mathcal{A}|}{2}}.
\]

Thus, $\Gamma=\sqrt{\frac{|\mathcal{A}|}{2}}$ is a valid upper bound on the information-ratio coefficient $\Gamma_t$ for every time step $t$. Therefore, by Theorem~\ref{thm1},
\begin{equation}
\mathbb{E}[\mathrm{Reg}_T]
\leq
\Gamma \sqrt{ H(\Pi^\star) \, T }
=
\sqrt{\frac{|\mathcal{A}|}{2} \, H(\Pi^\star) \, T }.
\end{equation}

Moreover, since $\Pi^\star$ maps the context space $\mathcal{C}$ to the action space $\mathcal{A}$, the total number of possible mappings is at most $|\mathcal{A}|^{|\mathcal{C}|}$ when $\mathcal{C}$ is finite. This implies that the entropy is bounded as $H(\Pi^\star) \leq \log\!\left(|\mathcal{A}|^{|\mathcal{C}|}\right) = |\mathcal{C}| \log\!\left(|\mathcal{A}|\right)$. Consequently,
\[
\mathbb{E}[\mathrm{Reg}_T]
\leq
\sqrt{\frac{T|\mathcal{A}||\mathcal{C}|\log|\mathcal{A}|}{2}}.
\]
This completes the proof of Corollary \ref{crr_reg}.

\subsection{Proof of Lemma \ref{lemm_1}}

We fix time step $t$ and the realized context $\mathbf{C}_t=\mathbf{c}_t$. We begin with the definition of the information gain:
\[
g_t(a|\mathbf{c}_t)
=
D_{KL}\!\left(
P_t(\Pi^*(\mathbf{c}_t),Y_{a,t} \mid \mathbf{c}_t)
\,\middle\|\,
P_t(\Pi^*(\mathbf{c}_t)\mid \mathbf{c}_t) P_t(Y_{a,t} \mid \mathbf{c}_t)
\right).
\]

We compute each term separately.

\paragraph{Posterior of $\Pi^*(\mathbf{c}_t)$.}

By definition,
\[
P_t(\Pi^*(\mathbf{c}_t)=a^*\mid\mathbf{c}_t)
=
P(\Pi^*(\mathbf{c}_t)=a^* \mid \mathcal{F}_t,\mathbf{C}_t=\mathbf{c}_t).
\]

Using the law of total expectation over $\boldsymbol{\theta}$,
\[
P_t(\Pi^*(\mathbf{c}_t)=a^*\mid\mathbf{c}_t)
=
\mathbb{E}_{\boldsymbol{\theta}\sim P_t(\boldsymbol{\theta}\mid\mathbf{c}_t)}
\left[
P(\Pi^*(\mathbf{c}_t)=a^* \mid \boldsymbol{\theta},\mathbf{c}_t)
\right].
\]

Given $\boldsymbol{\theta}$ and context $\mathbf{c}_t$, the optimal action is deterministic, hence
\[
P(\Pi^*(\mathbf{c}_t)=a^* \mid \boldsymbol{\theta},\mathbf{c}_t)
=
\mathbbm{1}\left\{
a^* = \arg\max_{a\in\mathcal{A}}
\mathbb{E}[Y_{a,t} \mid \boldsymbol{\theta},\mathbf{c}_t]
\right\}.
\]

Substituting,
\[
P_t(\Pi^*(\mathbf{c}_t)=a^*\mid\mathbf{c}_t)
=
\mathbb{E}_{\boldsymbol{\theta}\sim P_t(\boldsymbol{\theta}\mid\mathbf{c}_t)}
\left[
\mathbbm{1}\left\{
a^* = \arg\max_{a\in\mathcal{A}}
\mathbb{E}[Y_{a,t} \mid \boldsymbol{\theta},\mathbf{c}_t]
\right\}
\right].
\]

\paragraph{Posterior of $Y_a$.}

By definition,
\[
P_t(Y_{a,t}=y \mid \mathbf{c}_t)
=
P(Y_{a,t}=y \mid \mathcal{F}_t,\mathbf{C}_t=\mathbf{c}_t).
\]

Conditioning on $\boldsymbol{\theta}$ gives
\[
P_t(Y_{a,t}=y \mid \mathbf{c}_t)
=
\mathbb{E}_{\boldsymbol{\theta}\sim P_t(\boldsymbol{\theta}\mid\mathbf{c}_t)}
\left[
P(Y_{a,t}=y \mid \boldsymbol{\theta},\mathbf{c}_t)
\right].
\]

\paragraph{Joint posterior.}

Using the law of total expectation,
\begin{align*}
P_t(\Pi^*(\mathbf{c}_t)=a^*,Y_{a,t}=y \mid \mathbf{c}_t)
&=
\mathbb{E}_{\boldsymbol{\theta}\sim P_t(\boldsymbol{\theta}\mid\mathbf{c}_t)}
\Big[
P(\Pi^*(\mathbf{c}_t)=a^*,Y_{a,t}=y
\mid \boldsymbol{\theta},\mathbf{c}_t)
\Big].
\end{align*}

Since $\Pi^*(\mathbf{c}_t)$ is deterministic given $\boldsymbol{\theta}$ and $\mathbf{c}_t$, we have
\begin{align*}
P_t(\Pi^*(\mathbf{c}_t)=a^*,Y_{a,t}=y \mid \mathbf{c}_t)
&=
\mathbb{E}_{\boldsymbol{\theta}\sim P_t(\boldsymbol{\theta}\mid\mathbf{c}_t)}
\Big[
P(\Pi^*(\mathbf{c}_t)=a^* \mid \boldsymbol{\theta},\mathbf{c}_t)
\\
&\qquad\qquad \times
P(Y_{a,t}=y \mid \boldsymbol{\theta},\mathbf{c}_t)
\Big].
\end{align*}

Substituting the indicator representation,
\begin{align*}
P_t(\Pi^*(\mathbf{c}_t)=a^*,Y_{a,t}=y \mid \mathbf{c}_t)
=
\mathbb{E}_{\boldsymbol{\theta}\sim P_t(\boldsymbol{\theta}\mid\mathbf{c}_t)}
\Big[
\mathbbm{1}\{a^* = \arg\max_{a\in\mathcal{A}}
\mathbb{E}[Y_{a,t} \mid \boldsymbol{\theta},\mathbf{c}_t]\}
\\
\times
P(Y_{a,t}=y \mid \boldsymbol{\theta},\mathbf{c}_t)
\Big].
\end{align*}

This establishes the expressions used in the KL divergence.

\paragraph{Expected instantaneous regret.}

By definition,
\[
\Delta_t(a|\mathbf{c}_t)
=
\mathbb{E}\!\left[
Y_{\Pi^*(\mathbf{c}_t),t} - Y_{a,t}
\mid \mathbf{C}_t=\mathbf{c}_t,\mathcal{F}_t
\right].
\]

Using the law of total expectation over $\boldsymbol{\theta}$,
\[
\Delta_t(a|\mathbf{c}_t)
=
\mathbb{E}_{\boldsymbol{\theta}\sim P_t(\boldsymbol{\theta}\mid\mathbf{c}_t)}
\left[
\mathbb{E}\!\left[
Y_{\Pi^*(\mathbf{c}_t),t} - Y_{a,t}
\mid \boldsymbol{\theta},\mathbf{c}_t
\right]
\right].
\]

Expanding the first term gives
\begin{align*}
\mathbb{E}[Y_{\Pi^*(\mathbf{c}_t),t} \mid \boldsymbol{\theta},\mathbf{c}_t]
&=
\sum_{a^*\in\mathcal{A}}
\mathbbm{1}\left\{
a^*=\arg\max_{a\in\mathcal{A}}
\mathbb{E}[Y_{a,t} \mid \boldsymbol{\theta},\mathbf{c}_t]
\right\}
\\
&\qquad\qquad
\times
P(Y_{a^*,t}=1 \mid \boldsymbol{\theta},\mathbf{c}_t).
\end{align*}

Also, since the rewards are Bernoulli,
\[
\mathbb{E}[Y_{a,t} \mid \boldsymbol{\theta},\mathbf{c}_t]
=
P(Y_{a,t}=1 \mid \boldsymbol{\theta},\mathbf{c}_t).
\]

Substituting both expressions,
\begin{align*}
\Delta_t(a|\mathbf{c}_t)
=
\mathbb{E}_{\boldsymbol{\theta}\sim P_t(\boldsymbol{\theta}\mid\mathbf{c}_t)}
\Bigg[
\sum_{a^*\in\mathcal{A}}
\mathbbm{1}\left\{
a^*=\arg\max_{a\in\mathcal{A}}
\mathbb{E}[Y_{a,t} \mid \boldsymbol{\theta},\mathbf{c}_t]
\right\}
\\
\times P(Y_{a^*,t}=1 \mid \boldsymbol{\theta},\mathbf{c}_t)
-
P(Y_{a,t}=1 \mid \boldsymbol{\theta},\mathbf{c}_t)
\Bigg].
\end{align*}

This completes the proof of Lemma \ref{lemm_1}.

\subsection{Proof of Lemma \ref{lemm_2}}

Fix time step $t$ and the realized context $\mathbf{C}_t=\mathbf{c}_t$. From Lemma~\ref{lemm_1}, define
\[
f(\boldsymbol{\theta})
:=
\sum_{a^* \in \mathcal{A}}
\mathbbm{1}\!\left\{
a^* = \arg \max_{a\in \mathcal{A}} \mathbb{E}[Y_{a,t} \mid \boldsymbol{\theta},\mathbf{c}_t]
\right\}
P(Y_{a^*,t}=1 \mid \boldsymbol{\theta},\mathbf{c}_t)
-
P(Y_{a,t}=1 \mid \boldsymbol{\theta},\mathbf{c}_t).
\]
Then
\[
\Delta_t(a \mid \mathbf{c}_t)
=
\mathbb{E}_{\boldsymbol{\theta} \sim P_t(\boldsymbol{\theta}\mid\mathbf{c}_t)}
\!\left[f(\boldsymbol{\theta})\right].
\]

Let $\boldsymbol{\theta}_1,\ldots,\boldsymbol{\theta}_N$ be i.i.d.\ samples from $P_t(\boldsymbol{\theta}\mid\mathbf{c}_t)$, and define the Monte Carlo estimator
\[
\widehat{\Delta}_t(a \mid \mathbf{c}_t)
=
\frac{1}{N}\sum_{i=1}^N f(\boldsymbol{\theta}_i).
\]

We now bound the range of $f(\boldsymbol{\theta})$. For any fixed $\boldsymbol{\theta}$, exactly one term in the indicator sum equals $1$ because ties are broken deterministically. Hence,
\[
\sum_{a^* \in \mathcal{A}}
\mathbbm{1}\{\cdots\}\,P(Y_{a^*,t}=1 \mid \boldsymbol{\theta},\mathbf{c}_t)
=
P\!\left(Y_{\Pi^*(\mathbf{c}_t),t}=1
\mid \boldsymbol{\theta},\mathbf{c}_t\right).
\]
Because $\Pi^*(\mathbf{c}_t)$ maximizes the expected reward under $\boldsymbol{\theta}$ and $\mathbf{c}_t$,
\[
P\!\left(Y_{\Pi^*(\mathbf{c}_t),t}=1
\mid \boldsymbol{\theta},\mathbf{c}_t\right)
\geq
P(Y_{a,t}=1 \mid \boldsymbol{\theta},\mathbf{c}_t).
\]
Since both probabilities belong to $[0,1]$, it follows that
\[
0 \leq f(\boldsymbol{\theta}) \leq 1.
\]

Define $X_i := f(\boldsymbol{\theta}_i)$. Then $X_1,\ldots,X_N$ are independent and satisfy $0\leq X_i\leq1$ almost surely, and
\[
\mathbb{E}[X_i] = \Delta_t(a \mid \mathbf{c}_t),
\qquad
\frac{1}{N}\sum_{i=1}^N X_i = \widehat{\Delta}_t(a \mid \mathbf{c}_t).
\]
Applying Hoeffding's inequality to the sample mean of bounded variables yields, for any $\epsilon>0$,
\[
\mathbb{P}\!\left(
\left|
\widehat{\Delta}_t(a \mid \mathbf{c}_t) - \Delta_t(a \mid \mathbf{c}_t)
\right|
\geq \epsilon
\right)
\leq
2\exp\!\left(-2N\epsilon^2\right).
\]

Setting the right-hand side equal to $\delta$, i.e.,
\[
\delta = 2\exp\!\left(-2N\epsilon^2\right),
\]
gives
\[
\epsilon
=
\sqrt{\frac{1}{2N}\log\!\left(\frac{2}{\delta}\right)}.
\]
Therefore, with probability at least $1-\delta$,
\[
\left|
\widehat{\Delta}_t(a \mid \mathbf{c}_t) - \Delta_t(a \mid \mathbf{c}_t)
\right|
\leq
\sqrt{\frac{1}{2N}\log\!\left(\frac{2}{\delta}\right)}
\leq
\sqrt{\frac{2}{N}\log\!\left(\frac{2}{\delta}\right)}.
\]
Hence,
\[
\mathbb{P}\!\left(
\left|
\widehat{\Delta}_t(a \mid \mathbf{c}_t) - \Delta_t(a \mid \mathbf{c}_t)
\right|
\leq
\sqrt{\frac{2}{N}\log\!\left(\frac{2}{\delta}\right)}
\right)
\geq 1-\delta.
\]
This completes the proof of Lemma~\ref{lemm_2}.

\subsection{Proof of Lemma~\ref{lemm_gt_conc}}
\label{concen_sec}

Fix a time step $t$, a realized context $\mathbf{C}_t=\mathbf{c}_t$, and an action $a\in\mathcal{A}$. Let $K:=|\mathcal{A}|$, $X:=\Pi^\star(\mathbf{c}_t)$, and $Z:=Y_a$. Throughout the proof, we condition on $\mathcal{F}_t$ and $\mathbf{C}_t=\mathbf{c}_t$, so that the posterior distribution is $P_t^{\mathbf{c}_t}(\boldsymbol{\theta})$.

Let $\boldsymbol{\theta}^{(1)},\ldots,\boldsymbol{\theta}^{(N)}$ be independent samples from $P_t^{\mathbf{c}_t}(\boldsymbol{\theta})$. For each sample, let
\[
X^{(n)}
:=
\arg\max_{a'\in\mathcal{A}}
\mu_{\boldsymbol{\theta}^{(n)}}(a',\mathbf{c}_t),
\qquad
q_n^a(y)
:=
P(Y_a=y\mid\boldsymbol{\theta}^{(n)},\mathbf{c}_t).
\]
The Monte Carlo estimates of the distributions of $X$, $Z$, and $(X,Z)$ are
\begin{align}
\widehat P_X(a^\star)
&:=
\frac{1}{N}\sum_{n=1}^N
\mathbbm{1}\{X^{(n)}=a^\star\},
\label{eq:pxhat-new}\\
\widehat P_Z(y)
&:=
\frac{1}{N}\sum_{n=1}^N q_n^a(y),
\label{eq:pzhat-new}\\
\widehat P_{XZ}(a^\star,y)
&:=
\frac{1}{N}\sum_{n=1}^N
\mathbbm{1}\{X^{(n)}=a^\star\}q_n^a(y).
\label{eq:pxzhat-new}
\end{align}
By Lemma~\ref{lemm_1}, each coordinate is an unbiased estimator of the corresponding posterior probability:
\[
P_X(a^\star)
=
P_t^{\mathbf{c}_t}\!\left(
\Pi^\star(\mathbf{c}_t)=a^\star
\right),
\qquad
P_Z(y)
=
P_t^{\mathbf{c}_t}(Y_a=y),
\]
and
\[
P_{XZ}(a^\star,y)
=
P_t^{\mathbf{c}_t}\!\left(
\Pi^\star(\mathbf{c}_t)=a^\star,Y_a=y
\right).
\]

Each summand in \eqref{eq:pxhat-new}--\eqref{eq:pxzhat-new} lies in $[0,1]$. Hence, for any fixed coordinate, Hoeffding's inequality gives
\[
\mathbb{P}\!\left(
|\widehat p-p|>\beta
\right)
\leq
2\exp(-2N\beta^2).
\]
There are $K$ coordinates in $P_X$, $2$ coordinates in $P_Z$, and $2K$ coordinates in $P_{XZ}$. A union bound over these $3K+2$ coordinates therefore gives
\[
\mathbb{P}\!\left(
\max |\widehat p-p|>\beta
\right)
\leq
2(3K+2)\exp(-2N\beta^2).
\]
Setting
\[
\beta=\beta_N(\delta)
:=
\sqrt{
\frac{1}{2N}
\log\!\left(
\frac{2(3K+2)}{\delta}
\right)
}
\]
shows that, with probability at least $1-\delta$, all coordinate-wise bounds hold simultaneously. On this event,
\begin{equation}
\|\widehat P_X-P_X\|_1
\leq K\beta_N(\delta),\qquad
\|\widehat P_Z-P_Z\|_1
\leq2\beta_N(\delta),\qquad
\|\widehat P_{XZ}-P_{XZ}\|_1
\leq2K\beta_N(\delta).
\label{eq:l1-new}
\end{equation}

We next recall the finite-alphabet entropy-continuity bound used below. Let $P$ and $Q$ be distributions on an alphabet of size $m\geq2$, and define
\[
\tau:=\frac{1}{2}\|P-Q\|_1.
\]
Consider a maximal coupling $(U,V)$ with marginals $P$ and $Q$. By the maximal coupling theorem, among all couplings of $P$ and $Q$ it maximizes $\mathbb{P}(U=V)$, or equivalently minimizes $\mathbb{P}(U\neq V)$, and satisfies
\[
\mathbb{P}(U\neq V)
=
\operatorname{TV}(P,Q)
=
\frac{1}{2}\|P-Q\|_1
=
\tau.
\]
Define
\[
E:=\mathbbm{1}\{U\neq V\}.
\]
Since $E$ is a deterministic function of $(U,V)$,
\[
H(E\mid U,V)=0.
\]
Thus, by the chain rule,
\[
H(U\mid V)
=
H(E\mid V)+H(U\mid E,V).
\]
Moreover,
\[
H(E\mid V)\leq H(E)=h_2(\tau),
\]
where
\[
h_2(u):=-u\log u-(1-u)\log(1-u),
\qquad
0\log0:=0.
\]
When $E=0$, we have $U=V$, so
\[
H(U\mid E=0,V)=0.
\]
When $E=1$, for each fixed $V=v$, the variable $U$ can take at most $m-1$ possible values. Hence
\[
H(U\mid E=1,V)\leq\log(m-1).
\]
Therefore,
\[
H(U\mid E,V)
\leq
\tau\log(m-1),
\]
and consequently
\[
H(U\mid V)
\leq
h_2(\tau)+\tau\log(m-1).
\]
By symmetry,
\[
H(V\mid U)
\leq
h_2(\tau)+\tau\log(m-1).
\]
Since
\[
H(U)-H(V)
=
H(U\mid V)-H(V\mid U),
\]
and both conditional entropies are nonnegative,
\[
|H(U)-H(V)|
\leq
\max\{H(U\mid V),H(V\mid U)\}.
\]
Using $H(U)=H(P)$ and $H(V)=H(Q)$ gives
\[
|H(P)-H(Q)|
\leq
h_2(\tau)+\tau\log(m-1).
\]

For a bound $\|P-Q\|_1\leq\epsilon$, define
\begin{equation}
\phi_m(\epsilon)
:=
h_2\!\left(
\min\left\{\frac{\epsilon}{2},1-\frac{1}{m}\right\}
\right)
+
\min\left\{\frac{\epsilon}{2},1-\frac{1}{m}\right\}
\log(m-1).
\label{eq:phi-new}
\end{equation}
Let
\[
f_m(u):=h_2(u)+u\log(m-1).
\]
Since
\[
f_m'(u)
=
\log\!\left(
\frac{(m-1)(1-u)}{u}
\right),
\]
the function $f_m$ is increasing on
$[0,1-1/m]$ and satisfies
\[
f_m\!\left(1-\frac{1}{m}\right)=\log m.
\]
If $\epsilon/2\leq1-1/m$, then
$\tau\leq\epsilon/2$, and monotonicity gives
\[
|H(P)-H(Q)|
\leq
f_m(\tau)
\leq
f_m(\epsilon/2)
=
\phi_m(\epsilon).
\]
If $\epsilon/2>1-1/m$, then the trivial bound
\[
|H(P)-H(Q)|\leq\log m
\]
applies, while
\[
\phi_m(\epsilon)
=
f_m\!\left(1-\frac{1}{m}\right)
=
\log m.
\]
Therefore,
\begin{equation}
\|P-Q\|_1\leq\epsilon
\quad\Longrightarrow\quad
|H(P)-H(Q)|
\leq
\phi_m(\epsilon).
\label{eq:entropy-cont-new}
\end{equation}

Since $X=\Pi^\star(\mathbf{c}_t)$ takes values in $\mathcal{A}$, its alphabet has size at most $K$. Since $Z=Y_a$ is binary, its alphabet has size at most $2$, and the joint variable $(X,Z)$ has alphabet size at most $2K$. By the entropy representation of mutual information,
\[
g_t(a\mid\mathbf{c}_t)
=
H(P_X)+H(P_Z)-H(P_{XZ}),
\]
while
\[
\widehat g_t(a\mid\mathbf{c}_t)
=
H(\widehat P_X)+H(\widehat P_Z)-H(\widehat P_{XZ}).
\]
Therefore, by the triangle inequality,
\begin{align}
\left|
\widehat g_t(a\mid\mathbf{c}_t)
-
g_t(a\mid\mathbf{c}_t)
\right|
&\leq
|H(\widehat P_X)-H(P_X)|
+
|H(\widehat P_Z)-H(P_Z)|
\nonumber\\
&\qquad+
|H(\widehat P_{XZ})-H(P_{XZ})|.
\label{eq:mi-entropy-new}
\end{align}
Combining \eqref{eq:l1-new} and \eqref{eq:entropy-cont-new} yields
\[
|H(\widehat P_X)-H(P_X)|
\leq
\phi_K\!\left(K\beta_N(\delta)\right),
\]
\[
|H(\widehat P_Z)-H(P_Z)|
\leq
\phi_2\!\left(2\beta_N(\delta)\right),
\]
and
\[
|H(\widehat P_{XZ})-H(P_{XZ})|
\leq
\phi_{2K}\!\left(2K\beta_N(\delta)\right).
\]
Substituting these bounds into \eqref{eq:mi-entropy-new} gives
\[
\left|
\widehat g_t(a\mid\mathbf{c}_t)
-
g_t(a\mid\mathbf{c}_t)
\right|
\leq
r_N^g,
\]
where
\[
r_N^g
=
\phi_K\!\left(K\beta_N(\delta)\right)
+
\phi_2\!\left(2\beta_N(\delta)\right)
+
\phi_{2K}\!\left(2K\beta_N(\delta)\right).
\]
This establishes the claimed bound with probability at least $1-\delta$.

Finally, mutual information is nonnegative, and since $Y_a$ is binary,
\[
g_t(a\mid\mathbf{c}_t)
=
I_t\!\left(
\Pi^\star(\mathbf{c}_t);Y_a
\mid\mathbf{C}_t=\mathbf{c}_t
\right)
\leq
H_t(Y_a\mid\mathbf{C}_t=\mathbf{c}_t)
\leq
\log 2.
\]
Consequently,
\[
\underline g_t(a\mid\mathbf{c}_t)
=
\max\left\{
0,
\widehat g_t(a\mid\mathbf{c}_t)-r_N^g
\right\}
\]
and
\[
\overline g_t(a\mid\mathbf{c}_t)
=
\min\left\{
\log 2,
\widehat g_t(a\mid\mathbf{c}_t)+r_N^g
\right\}
\]
form a valid confidence interval for $g_t(a\mid\mathbf{c}_t)$.

Moreover, for fixed $K$ and $\delta$,
\[
\beta_N(\delta)\to0
\qquad
\text{as }N\to\infty.
\]
Since $\phi_m(\epsilon)\to0$ as $\epsilon\to0$ for every fixed finite $m$, it follows that
\[
r_N^g\to0
\qquad
\text{as }N\to\infty.
\]
Thus, the confidence radius, and hence the width of the confidence interval, vanishes as $N\to\infty$, completing the proof.

\subsection{Proof of Theorem~\ref{thm2} (Contextual Case)}

Fix a horizon $T\in\mathbb{N}$. Recall that $\mathcal{F}_t$ denotes
the history available immediately before observing the context at round
$t$. Thus,
\[
\mathcal{F}_t
=
\sigma\!\left(
\mathbf{C}_1,A_1,\mathbf{V}_{A_1,1},Y_{A_1,1},
\ldots,
\mathbf{C}_{t-1},A_{t-1},
\mathbf{V}_{A_{t-1},t-1},Y_{A_{t-1},t-1}
\right).
\]
After observing the current context $\mathbf{C}_t$, define the
pre-action sigma-field
\begin{equation}
\mathcal{B}_t
:=
\sigma(\mathcal{F}_t,\mathbf{C}_t).
\label{eq:def-Gt}
\end{equation}
Hence $\mathcal{B}_t$ contains exactly the environmental information
available to the learner at round $t$ after observing the current
context and before carrying out the fresh algorithmic randomization
used for action selection.

We condition only on the sequentially available information
$\mathcal{B}_t$. In particular, we never condition on the future
contexts $\mathbf{C}_{t+1},\ldots,\mathbf{C}_T$.

Let $U_t$ denote the fresh algorithmic randomization used at round $t$
before the final action draw, including the randomness used to generate
the Monte Carlo posterior samples and to construct
$\widehat{\pi}_t(\mathbf{C}_t)$. Conditional on $\mathcal{B}_t$, this
randomization is generated independently of the true SCM and the
round-$t$ potential outcomes. Thus,
\begin{equation}
U_t
\perp\!\!\!\perp
\left(
\boldsymbol{\theta},
\Pi^\star,
\{Y_{a,t}:a\in\mathcal{A}\}
\right)
\mid
\mathcal{B}_t.
\label{eq:Ut-independence}
\end{equation}
Define
\[
\mathcal{H}_t
:=
\sigma(\mathcal{B}_t,U_t).
\]
Then $\widehat{\pi}_t(\mathbf{C}_t)$ is
$\mathcal{H}_t$-measurable. Given $\mathcal{H}_t$, the final action is
sampled according to
\begin{equation}
\mathbb{P}(A_t=a\mid\mathcal{H}_t)
=
\widehat{\pi}_t(a\mid\mathbf{C}_t),
\qquad a\in\mathcal{A},
\label{eq:action-policy-Ht}
\end{equation}
using fresh action-sampling randomness independent of the true
environment.

Since both the construction of $\widehat{\pi}_t$ and the subsequent
action sampling use only $\mathcal{B}_t$ and fresh algorithmic
randomness independent of the environment, we have
\begin{equation}
A_t
\perp\!\!\!\perp
\left(
\Pi^\star,
\{Y_{a,t}:a\in\mathcal{A}\}
\right)
\mid
\mathcal{B}_t.
\label{eq:At-environment-independence}
\end{equation}
In particular,
\begin{equation}
A_t
\perp\!\!\!\perp
\Pi^\star
\mid
\mathcal{B}_t,
\label{eq:At-independence-Gt}
\end{equation}
and therefore
\begin{equation}
I\!\left(
\Pi^\star;A_t
\middle|
\mathcal{B}_t
\right)
=0.
\label{eq:At-MI-zero}
\end{equation}

For every realized context $\mathbf{C}_t=\mathbf{c}_t$ and every
$\pi\in S_{|\mathcal{A}|}$, define
\begin{align}
\Delta_t(\pi\mid\mathbf{c}_t)
&:=
\sum_{a\in\mathcal{A}}
\pi(a)\Delta_t(a\mid\mathbf{c}_t),
\\
g_t(\pi\mid\mathbf{c}_t)
&:=
\sum_{a\in\mathcal{A}}
\pi(a)g_t(a\mid\mathbf{c}_t),
\end{align}
where
\begin{align}
\Delta_t(a\mid\mathbf{c}_t)
&:=
\mathbb{E}\!\left[
Y_{\Pi^\star(\mathbf{c}_t),t}-Y_{a,t}
\middle|
\mathcal{F}_t,\mathbf{C}_t=\mathbf{c}_t
\right],
\label{eq:delta-local-proof}
\\
g_t(a\mid\mathbf{c}_t)
&:=
I\!\left(
\Pi^\star(\mathbf{c}_t);
Y_{a,t}
\middle|
\mathcal{F}_t,\mathbf{C}_t=\mathbf{c}_t
\right).
\label{eq:g-local-proof}
\end{align}
The corresponding information ratio is
\begin{equation}
\Psi_t(\pi\mid\mathbf{c}_t)
:=
\frac{
\Delta_t(\pi\mid\mathbf{c}_t)^2
}{
g_t(\pi\mid\mathbf{c}_t)
},
\label{eq:information-ratio-contextual}
\end{equation}
where we use the conventions
\[
0/0:=0,
\qquad
x/0:=+\infty
\quad\text{for }x>0.
\]

Let
\begin{equation}
\pi_t^\circ(\mathbf{c}_t)
\in
\arg\min_{\pi\in S_{|\mathcal{A}|}}
\Psi_t(\pi\mid\mathbf{c}_t)
\label{eq:oracle-ids-contextual}
\end{equation}
denote the oracle IDS distribution, and let
$\widehat{\pi}_t(\mathbf{c}_t)$ denote the distribution constructed by
Algorithm~\ref{alg: ids} from its Monte Carlo quantities.

\paragraph{Step 1: Sequential regret decomposition.}

Conditional on $\mathcal{H}_t$, the distribution
$\widehat{\pi}_t(\mathbf{C}_t)$ is fixed and
$A_t\sim\widehat{\pi}_t(\mathbf{C}_t)$. Moreover,
\eqref{eq:Ut-independence} implies that conditioning additionally on
$U_t$ does not change the conditional law of the true environment or
the round-$t$ potential outcomes. Hence
\begin{align}
&\mathbb{E}\!\left[
Y_{\Pi^\star(\mathbf{C}_t),t}-Y_{A_t,t}
\middle|
\mathcal{H}_t
\right]
\nonumber\\
&\qquad=
\sum_{a\in\mathcal{A}}
\widehat{\pi}_t(a\mid\mathbf{C}_t)
\,
\mathbb{E}\!\left[
Y_{\Pi^\star(\mathbf{C}_t),t}-Y_{a,t}
\middle|
\mathcal{H}_t
\right]
\nonumber\\
&\qquad=
\sum_{a\in\mathcal{A}}
\widehat{\pi}_t(a\mid\mathbf{C}_t)
\,
\Delta_t(a\mid\mathbf{C}_t)
\nonumber\\
&\qquad=
\Delta_t\!\left(
\widehat{\pi}_t(\mathbf{C}_t)
\mid
\mathbf{C}_t
\right).
\label{eq:conditional-regret-identity}
\end{align}
Taking expectations and using the tower property gives
\begin{equation}
\mathbb{E}\!\left[
Y_{\Pi^\star(\mathbf{C}_t),t}-Y_{A_t,t}
\right]
=
\mathbb{E}\!\left[
\Delta_t\!\left(
\widehat{\pi}_t(\mathbf{C}_t)
\mid
\mathbf{C}_t
\right)
\right].
\label{eq:single-round-regret-identity}
\end{equation}
Summing over $t=1,\ldots,T$ therefore yields
\begin{align}
&\mathbb{E}\!\left[
\sum_{t=1}^{T}
\left(
Y_{\Pi^\star(\mathbf{C}_t),t}
-
Y_{A_t,t}
\right)
\right]
\nonumber\\
&\qquad=
\mathbb{E}\!\left[
\sum_{t=1}^{T}
\Delta_t\!\left(
\widehat{\pi}_t(\mathbf{C}_t)
\mid
\mathbf{C}_t
\right)
\right].
\label{eq:regret-identity-contextual}
\end{align}

Whenever
\[
g_t\!\left(
\widehat{\pi}_t(\mathbf{C}_t)
\mid
\mathbf{C}_t
\right)>0,
\]
the definition of the information ratio gives
\begin{align}
\Delta_t\!\left(
\widehat{\pi}_t(\mathbf{C}_t)
\mid
\mathbf{C}_t
\right)
&=
\sqrt{
\Psi_t\!\left(
\widehat{\pi}_t(\mathbf{C}_t)
\mid
\mathbf{C}_t
\right)
}
\nonumber\\
&\qquad\times
\sqrt{
g_t\!\left(
\widehat{\pi}_t(\mathbf{C}_t)
\mid
\mathbf{C}_t
\right)
}.
\label{eq:delta-psi-g}
\end{align}
If both the numerator and denominator are zero, the corresponding
expected-regret contribution is zero under the convention $0/0=0$.
If the information gain is zero while the expected regret is positive,
the information ratio is $+\infty$, in which case the resulting
information-ratio bound is vacuous. Thus \eqref{eq:delta-psi-g}
covers every non-vacuous case.

\paragraph{Step 2: Bounding the cumulative information gain.}

For each round $t$, define
\begin{equation}
Z_t
:=
\Pi^\star(\mathbf{C}_t).
\label{eq:def-Zt}
\end{equation}
Because $\mathbf{C}_t$ is $\mathcal{B}_t$-measurable, once
$\mathcal{B}_t$ is fixed the random variable $Z_t$ is a deterministic
function of the full oracle mapping $\Pi^\star$. Hence the conditional
data-processing inequality implies, for every $a\in\mathcal{A}$,
\begin{equation}
I\!\left(
Z_t;Y_{a,t}
\middle|
\mathcal{B}_t
\right)
\leq
I\!\left(
\Pi^\star;Y_{a,t}
\middle|
\mathcal{B}_t
\right).
\label{eq:data-processing-contextual}
\end{equation}

Let
\begin{equation}
q_t(a)
:=
\mathbb{P}(A_t=a\mid\mathcal{B}_t).
\label{eq:def-qt}
\end{equation}
Since $\widehat{\pi}_t$ is constructed using fresh Monte Carlo
randomness after conditioning on $\mathcal{B}_t$,
\begin{equation}
q_t(a)
=
\mathbb{E}\!\left[
\widehat{\pi}_t(a\mid\mathbf{C}_t)
\middle|
\mathcal{B}_t
\right].
\label{eq:q-average-policy}
\end{equation}

By \eqref{eq:At-environment-independence}, conditioning on $A_t=a$
does not change the conditional joint law of $\Pi^\star$ and the
potential reward corresponding to action $a$. Hence
\begin{equation}
\mathcal{L}\!\left(
\Pi^\star,Y_{a,t}
\middle|
\mathcal{B}_t,A_t=a
\right)
=
\mathcal{L}\!\left(
\Pi^\star,Y_{a,t}
\middle|
\mathcal{B}_t
\right).
\label{eq:action-does-not-change-law}
\end{equation}

Throughout the remainder of this step, conditional mutual information
$I(X;Y\mid\mathcal{G})$ and conditional entropy
$H(X\mid\mathcal{G})$ denote the usual scalar conditional-information
quantities, i.e., they are averaged over the randomness in the
conditioning sigma-field $\mathcal{G}$.

Since $Z_t=\Pi^\star(\mathbf{C}_t)$ and $\mathbf{C}_t$ is
$\mathcal{B}_t$-measurable, \eqref{eq:action-does-not-change-law}
and the law of total expectation give
\begin{equation}
\mathbb{E}\!\left[
g_t\!\left(
\widehat{\pi}_t(\mathbf{C}_t)
\mid
\mathbf{C}_t
\right)
\right]
=
I\!\left(
Z_t;
Y_{A_t,t}
\middle|
\mathcal{B}_t,A_t
\right).
\label{eq:g-equals-MI}
\end{equation}

Applying conditional data processing, since $Z_t$ is a deterministic
function of $(\Pi^\star,\mathbf{C}_t)$ and $\mathbf{C}_t$ is contained
in $\mathcal{B}_t$, gives
\begin{equation}
I\!\left(
Z_t;
Y_{A_t,t}
\middle|
\mathcal{B}_t,A_t
\right)
\leq
I\!\left(
\Pi^\star;
Y_{A_t,t}
\middle|
\mathcal{B}_t,A_t
\right).
\label{eq:data-processing-random-action}
\end{equation}

By the chain rule,
\begin{align}
I\!\left(
\Pi^\star;
A_t,Y_{A_t,t}
\middle|
\mathcal{B}_t
\right)
&=
I\!\left(
\Pi^\star;
A_t
\middle|
\mathcal{B}_t
\right)
\nonumber\\
&\quad+
I\!\left(
\Pi^\star;
Y_{A_t,t}
\middle|
\mathcal{B}_t,A_t
\right).
\end{align}
Using \eqref{eq:At-MI-zero}, this yields
\begin{equation}
I\!\left(
\Pi^\star;
Y_{A_t,t}
\middle|
\mathcal{B}_t,A_t
\right)
=
I\!\left(
\Pi^\star;
A_t,Y_{A_t,t}
\middle|
\mathcal{B}_t
\right).
\label{eq:add-action}
\end{equation}
Consequently,
\begin{equation}
\mathbb{E}\!\left[
g_t\!\left(
\widehat{\pi}_t(\mathbf{C}_t)
\mid
\mathbf{C}_t
\right)
\right]
\leq
I\!\left(
\Pi^\star;
A_t,Y_{A_t,t}
\middle|
\mathcal{B}_t
\right).
\label{eq:local-info-to-action-reward}
\end{equation}

The learner observes not only $Y_{A_t,t}$ but also the remaining
observed variables $\mathbf{V}_{A_t,t}$. By the chain rule,
\begin{align}
&I\!\left(
\Pi^\star;
A_t,\mathbf{V}_{A_t,t},Y_{A_t,t}
\middle|
\mathcal{B}_t
\right)
\nonumber\\
&\qquad=
I\!\left(
\Pi^\star;
A_t,Y_{A_t,t}
\middle|
\mathcal{B}_t
\right)
+
I\!\left(
\Pi^\star;
\mathbf{V}_{A_t,t}
\middle|
\mathcal{B}_t,A_t,Y_{A_t,t}
\right).
\end{align}
Since conditional mutual information is nonnegative,
\begin{equation}
I\!\left(
\Pi^\star;
A_t,Y_{A_t,t}
\middle|
\mathcal{B}_t
\right)
\leq
I\!\left(
\Pi^\star;
A_t,\mathbf{V}_{A_t,t},Y_{A_t,t}
\middle|
\mathcal{B}_t
\right).
\label{eq:reward-less-than-full}
\end{equation}

Since
\[
\mathcal{B}_t
=
\sigma(\mathcal{F}_t,\mathbf{C}_t),
\]
the chain rule gives
\begin{align}
&I\!\left(
\Pi^\star;
\mathbf{C}_t,A_t,\mathbf{V}_{A_t,t},Y_{A_t,t}
\middle|
\mathcal{F}_t
\right)
\nonumber\\
&\qquad=
I\!\left(
\Pi^\star;
\mathbf{C}_t
\middle|
\mathcal{F}_t
\right)
\nonumber\\
&\qquad\quad+
I\!\left(
\Pi^\star;
A_t,\mathbf{V}_{A_t,t},Y_{A_t,t}
\middle|
\mathcal{F}_t,\mathbf{C}_t
\right).
\label{eq:add-context-chain}
\end{align}
Since
\[
I\!\left(
\Pi^\star;
\mathbf{C}_t
\middle|
\mathcal{F}_t
\right)
\geq 0,
\]
we obtain
\begin{align}
&I\!\left(
\Pi^\star;
A_t,\mathbf{V}_{A_t,t},Y_{A_t,t}
\middle|
\mathcal{F}_t,\mathbf{C}_t
\right)
\nonumber\\
&\qquad\leq
I\!\left(
\Pi^\star;
\mathbf{C}_t,A_t,\mathbf{V}_{A_t,t},Y_{A_t,t}
\middle|
\mathcal{F}_t
\right).
\label{eq:add-context-inequality}
\end{align}

Define the complete observation collected during round $t$ as
\begin{equation}
\mathcal{O}_t
:=
\left(
\mathbf{C}_t,
A_t,
\mathbf{V}_{A_t,t},
Y_{A_t,t}
\right).
\label{eq:def-Ot}
\end{equation}
By definition of the sequential history,
\begin{equation}
\mathcal{F}_{t+1}
=
\sigma(\mathcal{F}_t,\mathcal{O}_t).
\label{eq:F-update}
\end{equation}
Therefore, by the conditional mutual-information identity,
\begin{align}
I\!\left(
\Pi^\star;
\mathcal{O}_t
\middle|
\mathcal{F}_t
\right)
&=
H\!\left(
\Pi^\star
\middle|
\mathcal{F}_t
\right)
-
H\!\left(
\Pi^\star
\middle|
\mathcal{F}_t,\mathcal{O}_t
\right)
\nonumber\\
&=
H\!\left(
\Pi^\star
\middle|
\mathcal{F}_t
\right)
-
H\!\left(
\Pi^\star
\middle|
\mathcal{F}_{t+1}
\right).
\label{eq:one-step-entropy-drop}
\end{align}

Combining
\eqref{eq:local-info-to-action-reward},
\eqref{eq:reward-less-than-full},
\eqref{eq:add-context-inequality}, and
\eqref{eq:one-step-entropy-drop} gives
\begin{align}
&\mathbb{E}\!\left[
g_t\!\left(
\widehat{\pi}_t(\mathbf{C}_t)
\mid
\mathbf{C}_t
\right)
\right]
\nonumber\\
&\qquad\leq
H\!\left(
\Pi^\star
\middle|
\mathcal{F}_t
\right)
-
H\!\left(
\Pi^\star
\middle|
\mathcal{F}_{t+1}
\right).
\label{eq:one-step-info-bound}
\end{align}

Summing \eqref{eq:one-step-info-bound} over $t=1,\ldots,T$ yields
\begin{align}
&\mathbb{E}\!\left[
\sum_{t=1}^{T}
g_t\!\left(
\widehat{\pi}_t(\mathbf{C}_t)
\mid
\mathbf{C}_t
\right)
\right]
\nonumber\\
&\qquad\leq
\sum_{t=1}^{T}
\left[
H\!\left(
\Pi^\star
\middle|
\mathcal{F}_t
\right)
-
H\!\left(
\Pi^\star
\middle|
\mathcal{F}_{t+1}
\right)
\right]
\nonumber\\
&\qquad=
H\!\left(
\Pi^\star
\middle|
\mathcal{F}_1
\right)
-
H\!\left(
\Pi^\star
\middle|
\mathcal{F}_{T+1}
\right)
\nonumber\\
&\qquad\leq
H(\Pi^\star).
\label{eq:total-info-final}
\end{align}
The middle equality follows because the conditional-entropy terms
telescope. Since $\mathcal{F}_1$ contains no observations,
\[
H(\Pi^\star\mid\mathcal{F}_1)=H(\Pi^\star),
\]
and conditional entropy is nonnegative, so
\[
H(\Pi^\star\mid\mathcal{F}_{T+1})\geq 0.
\]
Importantly, this cumulative information bound uses only the
sequential filtration and never conditions on future contexts.
\paragraph{Step 3: Oracle information-ratio bound.}

Fix a round $t$ and condition on the realized pre-action information
$\mathcal{B}_t$. Write $\mathbf{C}_t=\mathbf{c}_t$ and define
\begin{equation}
p_t(a\mid\mathbf{c}_t)
:=
\mathbb{P}\!\left(
\Pi^\star(\mathbf{c}_t)=a
\middle|
\mathcal{F}_t,\mathbf{C}_t=\mathbf{c}_t
\right).
\label{eq:posterior-opt-action}
\end{equation}
Consider the Thompson-sampling distribution
\begin{equation}
\pi_t^{\mathrm{TS}}(a\mid\mathbf{c}_t)
:=
p_t(a\mid\mathbf{c}_t).
\label{eq:TS-dist-proof}
\end{equation}
Since $\pi_t^\circ$ minimizes the true information ratio,
\begin{equation}
\Psi_t(\pi_t^\circ\mid\mathbf{c}_t)
\leq
\Psi_t(\pi_t^{\mathrm{TS}}\mid\mathbf{c}_t).
\label{eq:oracle-below-ts}
\end{equation}

For $a,b\in\mathcal{A}$, define
\begin{align}
\mu_{t,a}
&:=
\mathbb{E}\!\left[
Y_{a,t}
\middle|
\mathcal{F}_t,\mathbf{C}_t=\mathbf{c}_t
\right],
\\
\mu_{t,a\mid b}
&:=
\mathbb{E}\!\left[
Y_{a,t}
\middle|
\mathcal{F}_t,
\Pi^\star(\mathbf{c}_t)=b,
\mathbf{C}_t=\mathbf{c}_t
\right].
\end{align}

For the Thompson-sampling distribution,
\begin{align}
\Delta_t(\pi_t^{\mathrm{TS}}\mid\mathbf{c}_t)
&=
\sum_{a\in\mathcal{A}}
p_t(a\mid\mathbf{c}_t)
\Delta_t(a\mid\mathbf{c}_t)
\nonumber\\
&=
\sum_{a\in\mathcal{A}}
p_t(a\mid\mathbf{c}_t)
\left(
\mu_{t,a\mid a}-\mu_{t,a}
\right).
\label{eq:ts-regret-expansion}
\end{align}
The second equality follows from posterior matching: an independent
Thompson sample induces the same conditional distribution over the
optimal action as $\Pi^\star(\mathbf{c}_t)$.

Applying
\[
\left(\sum_{a=1}^{K}x_a\right)^2
\leq
K\sum_{a=1}^{K}x_a^2
\]
with
\[
x_a
=
p_t(a\mid\mathbf{c}_t)
\left(
\mu_{t,a\mid a}-\mu_{t,a}
\right)
\]
gives
\begin{align}
\Delta_t(\pi_t^{\mathrm{TS}}\mid\mathbf{c}_t)^2
&\leq
|\mathcal{A}|
\sum_{a\in\mathcal{A}}
p_t(a\mid\mathbf{c}_t)^2
\left(
\mu_{t,a\mid a}-\mu_{t,a}
\right)^2.
\label{eq:ts-cauchy}
\end{align}

Since $Y_{a,t}$ is Bernoulli,
\[
\left|
\mu_{t,a\mid a}-\mu_{t,a}
\right|
=
\left\|
P(
Y_{a,t}
\mid
\Pi^\star(\mathbf{c}_t)=a,\mathcal{B}_t
)
-
P(
Y_{a,t}
\mid
\mathcal{B}_t
)
\right\|_{\mathrm{TV}}.
\]
Pinsker's inequality therefore implies
\begin{align}
\left(
\mu_{t,a\mid a}-\mu_{t,a}
\right)^2
&\leq
\frac{1}{2}
D_{\mathrm{KL}}\!\left(
P(
Y_{a,t}
\mid
\mathcal{F}_t,
\Pi^\star(\mathbf{c}_t)=a,
\mathbf{C}_t=\mathbf{c}_t
)
\middle\|
P(
Y_{a,t}
\mid
\mathcal{F}_t,
\mathbf{C}_t=\mathbf{c}_t
)
\right).
\label{eq:ts-pinsker}
\end{align}
Substituting this into \eqref{eq:ts-cauchy} gives
\begin{align}
\Delta_t(\pi_t^{\mathrm{TS}}\mid\mathbf{c}_t)^2
&\leq
\frac{|\mathcal{A}|}{2}
\sum_{a\in\mathcal{A}}
p_t(a\mid\mathbf{c}_t)^2
\nonumber\\
&\qquad\times
D_{\mathrm{KL}}\!\left(
P(
Y_{a,t}
\mid
\mathcal{F}_t,
\Pi^\star(\mathbf{c}_t)=a,
\mathbf{C}_t=\mathbf{c}_t
)
\middle\|
P(
Y_{a,t}
\mid
\mathcal{F}_t,
\mathbf{C}_t=\mathbf{c}_t
)
\right).
\label{eq:ts-diagonal-kl}
\end{align}

On the other hand, using the KL representation of conditional mutual
information,
\begin{align}
g_t(\pi_t^{\mathrm{TS}}\mid\mathbf{c}_t)
&=
\sum_{a\in\mathcal{A}}
p_t(a\mid\mathbf{c}_t)
I\!\left(
\Pi^\star(\mathbf{c}_t);
Y_{a,t}
\middle|
\mathcal{B}_t
\right)
\nonumber\\
&=
\sum_{a\in\mathcal{A}}
\sum_{b\in\mathcal{A}}
p_t(a\mid\mathbf{c}_t)
p_t(b\mid\mathbf{c}_t)
\nonumber\\
&\qquad\times
D_{\mathrm{KL}}\!\left(
P(
Y_{a,t}
\mid
\mathcal{F}_t,
\Pi^\star(\mathbf{c}_t)=b,
\mathbf{C}_t=\mathbf{c}_t
)
\middle\|
P(
Y_{a,t}
\mid
\mathcal{F}_t,
\mathbf{C}_t=\mathbf{c}_t
)
\right).
\label{eq:ts-information-expansion}
\end{align}
Every term in \eqref{eq:ts-information-expansion} is nonnegative.
Retaining only the diagonal terms $b=a$ therefore gives
\begin{align}
g_t(\pi_t^{\mathrm{TS}}\mid\mathbf{c}_t)
&\geq
\sum_{a\in\mathcal{A}}
p_t(a\mid\mathbf{c}_t)^2
\nonumber\\
&\qquad\times
D_{\mathrm{KL}}\!\left(
P(
Y_{a,t}
\mid
\mathcal{F}_t,
\Pi^\star(\mathbf{c}_t)=a,
\mathbf{C}_t=\mathbf{c}_t
)
\middle\|
P(
Y_{a,t}
\mid
\mathcal{F}_t,
\mathbf{C}_t=\mathbf{c}_t
)
\right).
\label{eq:ts-information-diagonal}
\end{align}
Comparing \eqref{eq:ts-diagonal-kl} and
\eqref{eq:ts-information-diagonal} yields
\begin{equation}
\Delta_t(\pi_t^{\mathrm{TS}}\mid\mathbf{c}_t)^2
\leq
\frac{|\mathcal{A}|}{2}
g_t(\pi_t^{\mathrm{TS}}\mid\mathbf{c}_t).
\label{eq:TS-ratio-step}
\end{equation}
Therefore,
\begin{equation}
\Psi_t(\pi_t^\circ\mid\mathbf{c}_t)
\leq
\Psi_t(\pi_t^{\mathrm{TS}}\mid\mathbf{c}_t)
\leq
\frac{|\mathcal{A}|}{2}.
\label{eq:oracle-ratio-bound-contextual}
\end{equation}

\paragraph{Step 4: Monte Carlo confidence sets.}

For
$\boldsymbol{\Delta}\in[0,1]^{|\mathcal{A}|}$
and
$\boldsymbol{g}\in\mathbb{R}^{|\mathcal{A}|}_{+}$, define
\begin{equation}
F_t(
\pi;
\boldsymbol{\Delta},
\boldsymbol{g}
)
:=
\frac{
(\pi^\top\boldsymbol{\Delta})^2
}{
\pi^\top\boldsymbol{g}
}.
\label{eq:def-Ft-proof}
\end{equation}

At round $t$, Algorithm~\ref{alg: ids} selects
\[
\widetilde{\boldsymbol{\Delta}}_t(\mathbf{C}_t)
\in
\mathcal{C}^{\Delta}_t(\mathbf{C}_t)
\]
and
\[
\widetilde{\boldsymbol{g}}_t(\mathbf{C}_t)
\in
\mathcal{C}^{g}_t(\mathbf{C}_t),
\]
and defines
\begin{equation}
\widehat{\Psi}_t(\pi\mid\mathbf{C}_t)
:=
F_t\!\left(
\pi;
\widetilde{\boldsymbol{\Delta}}_t(\mathbf{C}_t),
\widetilde{\boldsymbol{g}}_t(\mathbf{C}_t)
\right).
\label{eq:optimistic-information-ratio}
\end{equation}
Accordingly,
\begin{equation}
\widehat{\pi}_t(\mathbf{C}_t)
\in
\arg\min_{\pi\in S_{|\mathcal{A}|}}
\widehat{\Psi}_t(\pi\mid\mathbf{C}_t).
\label{eq:algorithm-optimistic-policy}
\end{equation}

For each $t$, let $\mathcal{E}_t$ denote the event that, simultaneously
for every $a\in\mathcal{A}$ at the realized context,
\begin{align}
\Delta_t(a\mid\mathbf{C}_t)
&\in
\mathcal{C}^{\Delta}_{t,a}(\mathbf{C}_t),
\\
g_t(a\mid\mathbf{C}_t)
&\in
\mathcal{C}^{g}_{t,a}(\mathbf{C}_t).
\end{align}
Each of Lemmas~\ref{lemm_2} and~\ref{lemm_gt_conc} fails for a fixed
action with probability at most
\[
\delta
=
\frac{\delta'}{2T|\mathcal{A}|}.
\]
There are $|\mathcal{A}|$ actions and two confidence bounds per action.
Therefore, by a union bound,
\begin{align}
\mathbb{P}(\mathcal{E}_t^c)
&\leq
2|\mathcal{A}|\delta
\nonumber\\
&=
2|\mathcal{A}|
\frac{\delta'}{2T|\mathcal{A}|}
\nonumber\\
&=
\frac{\delta'}{T}.
\label{eq:round-good-event}
\end{align}
A second union bound over the $T$ rounds gives
\begin{align}
\mathbb{P}\!\left(
\bigcup_{t=1}^{T}\mathcal{E}_t^c
\right)
&\leq
\sum_{t=1}^{T}\mathbb{P}(\mathcal{E}_t^c)
\nonumber\\
&\leq
\sum_{t=1}^{T}\frac{\delta'}{T}
=
\delta'.
\end{align}
Equivalently,
\begin{equation}
\mathbb{P}\!\left(
\bigcap_{t=1}^{T}\mathcal{E}_t
\right)
\geq
1-\delta'.
\label{eq:good-event-probability}
\end{equation}

For every $t$, define
\begin{equation}
\gamma_t
:=
\sup_{\substack{
\pi\in S_{|\mathcal{A}|},\\
\boldsymbol{\Delta}_t^{\,1},
\boldsymbol{\Delta}_t^{\,2}
\in\mathcal{C}^{\Delta}_t(\mathbf{C}_t),\\
\boldsymbol{g}_t^{\,1},
\boldsymbol{g}_t^{\,2}
\in\mathcal{C}^{g}_t(\mathbf{C}_t)
}}
\left|
\frac{
(\pi^\top\boldsymbol{\Delta}_t^{\,1})^2
}{
\pi^\top\boldsymbol{g}_t^{\,1}
}
-
\frac{
(\pi^\top\boldsymbol{\Delta}_t^{\,2})^2
}{
\pi^\top\boldsymbol{g}_t^{\,2}
}
\right|
\in[0,+\infty].
\label{eq:gamma-t}
\end{equation}
We use the conventions $0/0:=0$ and $x/0:=+\infty$ for $x>0$.

On $\mathcal{E}_t$, the true pair
\[
\left(
\boldsymbol{\Delta}_t(\mathbf{C}_t),
\boldsymbol{g}_t(\mathbf{C}_t)
\right)
\]
and the Monte Carlo pair selected by Algorithm~\ref{alg: ids},
\[
\left(
\widetilde{\boldsymbol{\Delta}}_t(\mathbf{C}_t),
\widetilde{\boldsymbol{g}}_t(\mathbf{C}_t)
\right),
\]
belong to the same confidence product set. Hence
\begin{equation}
\left|
\Psi_t(\pi\mid\mathbf{C}_t)
-
\widehat{\Psi}_t(\pi\mid\mathbf{C}_t)
\right|
\leq
\gamma_t,
\qquad
\forall\pi\in S_{|\mathcal{A}|}.
\label{eq:true-vs-optimistic-ratio}
\end{equation}

Applying \eqref{eq:true-vs-optimistic-ratio} first to
$\widehat{\pi}_t$ and then to $\pi_t^\circ$, and using the minimizing
property of $\widehat{\pi}_t$, gives on $\mathcal{E}_t$
\begin{align}
\Psi_t(\widehat{\pi}_t\mid\mathbf{C}_t)
&\leq
\widehat{\Psi}_t(\widehat{\pi}_t\mid\mathbf{C}_t)
+\gamma_t
\nonumber\\
&\leq
\widehat{\Psi}_t(\pi_t^\circ\mid\mathbf{C}_t)
+\gamma_t
\nonumber\\
&\leq
\Psi_t(\pi_t^\circ\mid\mathbf{C}_t)
+2\gamma_t
\nonumber\\
&\leq
\frac{|\mathcal{A}|}{2}
+2\gamma_t.
\label{eq:algorithm-ratio-bound-contextual}
\end{align}
If an admissible information-gain denominator is zero while the
corresponding numerator is positive, then $\gamma_t=+\infty$, in which
case the resulting bound is valid but vacuous.

\paragraph{Step 5: Final bound.}

From \eqref{eq:single-round-regret-identity},
\[
\mathbb{E}\!\left[
Y_{\Pi^\star(\mathbf{C}_t),t}-Y_{A_t,t}
\right]
=
\mathbb{E}\!\left[
\Delta_t\!\left(
\widehat{\pi}_t(\mathbf{C}_t)
\mid
\mathbf{C}_t
\right)
\right].
\]
Splitting the right-hand side according to $\mathcal{E}_t$ gives
\begin{align}
&\mathbb{E}\!\left[
\Delta_t\!\left(
\widehat{\pi}_t(\mathbf{C}_t)
\mid
\mathbf{C}_t
\right)
\right]
\nonumber\\
&=
\mathbb{E}\!\left[
\mathbbm{1}_{\mathcal{E}_t}
\Delta_t\!\left(
\widehat{\pi}_t(\mathbf{C}_t)
\mid
\mathbf{C}_t
\right)
\right]
\nonumber\\
&\qquad+
\mathbb{E}\!\left[
\mathbbm{1}_{\mathcal{E}_t^c}
\Delta_t\!\left(
\widehat{\pi}_t(\mathbf{C}_t)
\mid
\mathbf{C}_t
\right)
\right].
\end{align}
Since rewards lie in $[0,1]$,
\[
0
\leq
\Delta_t\!\left(
\widehat{\pi}_t(\mathbf{C}_t)
\mid
\mathbf{C}_t
\right)
\leq1.
\]
Therefore,
\begin{align}
&\mathbb{E}\!\left[
\mathbbm{1}_{\mathcal{E}_t^c}
\Delta_t\!\left(
\widehat{\pi}_t(\mathbf{C}_t)
\mid
\mathbf{C}_t
\right)
\right]
\nonumber\\
&\qquad\leq
\mathbb{E}\!\left[
\mathbbm{1}_{\mathcal{E}_t^c}
\right]
=
\mathbb{P}(\mathcal{E}_t^c)
\leq
\frac{\delta'}{T}.
\label{eq:bad-event-regret-round}
\end{align}
Hence
\begin{align}
&\mathbb{E}\!\left[
Y_{\Pi^\star(\mathbf{C}_t),t}-Y_{A_t,t}
\right]
\nonumber\\
&\qquad\leq
\mathbb{E}\!\left[
\mathbbm{1}_{\mathcal{E}_t}
\Delta_t\!\left(
\widehat{\pi}_t(\mathbf{C}_t)
\mid
\mathbf{C}_t
\right)
\right]
+
\frac{\delta'}{T}.
\label{eq:round-regret-split}
\end{align}

Summing over all $T$ rounds therefore gives
\begin{align}
&\mathbb{E}\!\left[
\sum_{t=1}^{T}
\left(
Y_{\Pi^\star(\mathbf{C}_t),t}
-
Y_{A_t,t}
\right)
\right]
\nonumber\\
&\qquad\leq
\mathbb{E}\!\left[
\sum_{t=1}^{T}
\mathbbm{1}_{\mathcal{E}_t}
\Delta_t\!\left(
\widehat{\pi}_t(\mathbf{C}_t)
\mid
\mathbf{C}_t
\right)
\right]
+
\sum_{t=1}^{T}\frac{\delta'}{T}
\nonumber\\
&\qquad=
\mathbb{E}\!\left[
\sum_{t=1}^{T}
\mathbbm{1}_{\mathcal{E}_t}
\Delta_t\!\left(
\widehat{\pi}_t(\mathbf{C}_t)
\mid
\mathbf{C}_t
\right)
\right]
+
\delta'.
\label{eq:regret-good-bad}
\end{align}

On $\mathcal{E}_t$, the information-ratio factorization gives
\[
\Delta_t\!\left(
\widehat{\pi}_t(\mathbf{C}_t)
\mid
\mathbf{C}_t
\right)
=
\sqrt{
\Psi_t\!\left(
\widehat{\pi}_t(\mathbf{C}_t)
\mid
\mathbf{C}_t
\right)
}
\sqrt{
g_t\!\left(
\widehat{\pi}_t(\mathbf{C}_t)
\mid
\mathbf{C}_t
\right)
}.
\]
Thus, by Cauchy--Schwarz over $t$ and then over the underlying
probability space,
\begin{align}
&\mathbb{E}\!\left[
\sum_{t=1}^{T}
\mathbbm{1}_{\mathcal{E}_t}
\Delta_t\!\left(
\widehat{\pi}_t(\mathbf{C}_t)
\mid
\mathbf{C}_t
\right)
\right]
\nonumber\\
&\qquad\leq
\sqrt{
\mathbb{E}\!\left[
\sum_{t=1}^{T}
\mathbbm{1}_{\mathcal{E}_t}
\Psi_t\!\left(
\widehat{\pi}_t(\mathbf{C}_t)
\mid
\mathbf{C}_t
\right)
\right]
}
\nonumber\\
&\qquad\qquad\times
\sqrt{
\mathbb{E}\!\left[
\sum_{t=1}^{T}
\mathbbm{1}_{\mathcal{E}_t}
g_t\!\left(
\widehat{\pi}_t(\mathbf{C}_t)
\mid
\mathbf{C}_t
\right)
\right]
}.
\label{eq:good-event-CS-final}
\end{align}

Since $g_t\geq0$, \eqref{eq:total-info-final} implies
\begin{equation}
\mathbb{E}\!\left[
\sum_{t=1}^{T}
\mathbbm{1}_{\mathcal{E}_t}
g_t\!\left(
\widehat{\pi}_t(\mathbf{C}_t)
\mid
\mathbf{C}_t
\right)
\right]
\leq
H(\Pi^\star).
\label{eq:good-info-bound}
\end{equation}
Moreover, \eqref{eq:algorithm-ratio-bound-contextual} gives
\begin{align}
&\mathbb{E}\!\left[
\sum_{t=1}^{T}
\mathbbm{1}_{\mathcal{E}_t}
\Psi_t\!\left(
\widehat{\pi}_t(\mathbf{C}_t)
\mid
\mathbf{C}_t
\right)
\right]
\nonumber\\
&\qquad\leq
\mathbb{E}\!\left[
\sum_{t=1}^{T}
\mathbbm{1}_{\mathcal{E}_t}
\left(
\frac{|\mathcal{A}|}{2}
+
2\gamma_t
\right)
\right]
\nonumber\\
&\qquad\leq
\frac{T|\mathcal{A}|}{2}
+
2\sum_{t=1}^{T}
\mathbb{E}\!\left[
\gamma_t\mathbbm{1}_{\mathcal{E}_t}
\right].
\label{eq:good-ratio-bound}
\end{align}

Combining
\eqref{eq:good-event-CS-final},
\eqref{eq:good-info-bound}, and
\eqref{eq:good-ratio-bound}, and then substituting into
\eqref{eq:regret-good-bad}, yields
\begin{align}
&\mathbb{E}\!\left[
\sum_{t=1}^{T}
\left(
Y_{\Pi^\star(\mathbf{C}_t),t}
-
Y_{A_t,t}
\right)
\right]
\nonumber\\
&\qquad\leq
\sqrt{
H(\Pi^\star)
\left(
\frac{T|\mathcal{A}|}{2}
+
2\sum_{t=1}^{T}
\mathbb{E}\!\left[
\gamma_t\mathbbm{1}_{\mathcal{E}_t}
\right]
\right)
}
+
\delta'.
\label{eq:final-thm2-contextual}
\end{align}

Finally,
\[
\mathbb{P}\!\left(
\bigcap_{t=1}^{T}\mathcal{E}_t
\right)
\geq1-\delta',
\]
so all Monte Carlo confidence intervals simultaneously contain their
true quantities over all rounds and actions with probability at least
$1-\delta'$. The proof uses only the sequential filtration and never
conditions on future contexts. This completes the proof.

\subsection{SCM Specifications for the Synthetic Experiments}
\label{app:experimental_scms}

For reproducibility, we provide the complete parameterization of the SCMs
used in the experiments. For each task, a single fixed SCM is used across
all algorithms and independent trials. The learner is given the causal graph
and feasible intervention set, but not the numerical values of the CPTs
reported below. The causal methods instead learn these mechanisms
sequentially from the observations generated during interaction.

All variables below are discrete. The reward is binary,
$Y\in\{0,1\}$, and we report only the probability of value $1$ for each
binary CPT row; the probability of value $0$ follows by complementation.
Tasks 2 and 3 use a three-valued context $C\in\{0,1,2\}$ that is observed
before action selection.

\paragraph{Task 1.}
The causal graph factorizes as
\[
P(z,x,y)
=
P(z)P(x\mid z)P(y\mid x,z),
\]
where $Z$ is non-manipulable. The POMIS-induced intervention set is
\[
\mathcal A
=
\{do(),do(X=0),do(X=1)\}.
\]
The fixed SCM used in the experiment is
\[
P(Z=1)=0.1235637335,
\]
with
\[
\begin{array}{c|cc}
z & 0 & 1\\
\hline
P(X=1\mid Z=z)
&0.1895847557&0.5065343149
\end{array},
\]
and reward CPT
\[
\begin{array}{cc|c}
X & Z & P(Y=1\mid X,Z)\\
\hline
0&0&0.4397208962\\
0&1&0.1826673962\\
1&0&0.5313639348\\
1&1&0.6778901277
\end{array}.
\]

\paragraph{Task 2.}
The observational distribution corresponding to the graph in
Figure~\ref{exp_graphs}(b) factorizes as
\[
P(c,z,x_1,x_2,y)
=
P(c)P(z)P(x_1\mid c)
P(x_2\mid c,z,x_1)
P(y\mid x_2,z).
\]
Here $Z \in \{0,1\}$ is non-manipulable and $C\in\{0,1,2\}$ is observed before
intervention selection. The intervention set is
\[
\mathcal A
=
\{
do(X_1=0),do(X_1=1),
do(X_2=0),do(X_2=1)
\}.
\]

The context distribution is
\[
\begin{array}{c|ccc}
c&0&1&2\\
\hline
P(C=c)
&0.2638117514
&0.3618275639
&0.3743606847
\end{array},
\]
and
\[
P(Z=1)=0.6270600764.
\]

The $X_1$ mechanism is
\[
\begin{array}{c|ccc}
c&0&1&2\\
\hline
P(X_1=1\mid C=c)
&0.2895725473
&0.1108342648
&0.1115481692
\end{array}.
\]

The CPT for $X_2$ is
\[
\begin{array}{ccc|c}
C&Z&X_1&P(X_2=1\mid C,Z,X_1)\\
\hline
0&0&0&0.9241888693\\
0&0&1&0.8739701462\\
0&1&0&0.0612447247\\
0&1&1&0.7169166511\\
\hline
1&0&0&0.6815143210\\
1&0&1&0.2183190710\\
1&1&0&0.3463960432\\
1&1&1&0.8775868842\\
\hline
2&0&0&0.2757739657\\
2&0&1&0.7334004268\\
2&1&0&0.8266781099\\
2&1&1&0.0765423723
\end{array}.
\]

Finally, the reward mechanism is
\[
\begin{array}{cc|c}
X_2&Z&P(Y=1\mid X_2,Z)\\
\hline
0&0&0.7264079228\\
0&1&0.1002033444\\
1&0&0.1607285977\\
1&1&0.9202830276
\end{array}.
\]

\paragraph{Task 3.}
The observational distribution corresponding to the graph in
Figure~\ref{exp_graphs}(c) factorizes as
\[
\begin{split}
P(c,z_1,z_2,x_1,x_3,x_2,y)
={}&P(c)P(z_1)P(z_2)
P(x_1\mid z_1)P(x_3\mid c,z_2)\\
&\times P(x_2\mid x_3)
P(y\mid c,z_1,z_2,x_1,x_2).
\end{split}
\]
Here $Z_1$ and $Z_2$ are non-manipulable, and
$C\in\{0,1,2\}$ is observed before action selection. The POMISs
\[
\{\emptyset,\{X_1\},\{X_2\},\{X_1,X_2\}\}
\]
induce the nine interventions
\[
\begin{split}
\mathcal A=\{&
do(),do(X_1=0),do(X_1=1),
do(X_2=0),do(X_2=1),\\
&do(X_1=0,X_2=0),do(X_1=0,X_2=1),
do(X_1=1,X_2=0),do(X_1=1,X_2=1)\}.
\end{split}
\]

The context distribution is
\[
\begin{array}{c|ccc}
c&0&1&2\\
\hline
P(C=c)
&0.3947065369
&0.2930579390
&0.3122355241
\end{array}.
\]

The two non-manipulable variables $Z_1,Z_2\in\{0,1\}$ satisfy
\[
P(Z_1=1)=0.6157645285,
\qquad
P(Z_2=1)=0.4164692396.
\]

The reward mechanism depends on
$(C,Z_1,Z_2,X_1,X_2)$. Table~\ref{tab:task3_reward_cpt}
reports the complete CPT.

\begin{table}[h]
\centering
\caption{Reward CPT for Task 3. Each entry gives
$P(Y=1\mid C,Z_1,Z_2,X_1,X_2)$.}
\label{tab:task3_reward_cpt}
\small
\begin{tabular}{cccc|ccc}
\toprule
$Z_1$ & $Z_2$ & $X_1$ & $X_2$
& $C=0$ & $C=1$ & $C=2$\\
\midrule
0&0&0&0&0.417982&0.871489&0.633019\\
0&0&0&1&0.774177&0.482673&0.506749\\
0&0&1&0&0.339608&0.514955&0.466771\\
0&0&1&1&0.332413&0.757574&0.685939\\
0&1&0&0&0.946975&0.221492&0.683195\\
0&1&0&1&0.645632&0.777338&0.638943\\
0&1&1&0&0.906106&0.270157&0.004389\\
0&1&1&1&0.933135&0.132925&0.137802\\
1&0&0&0&0.213621&0.064126&0.887043\\
1&0&0&1&0.719798&0.884595&0.078232\\
1&0&1&0&0.939736&0.244386&0.404010\\
1&0&1&1&0.526923&0.186350&0.016627\\
1&1&0&0&0.085217&0.021726&0.620457\\
1&1&0&1&0.059653&0.438016&0.786023\\
1&1&1&0&0.906073&0.144986&0.851202\\
1&1&1&1&0.915141&0.280979&0.329918\\
\bottomrule
\end{tabular}
\end{table}

The remaining non-reward mechanisms are
\[
\begin{array}{c|cc}
z_1&0&1\\
\hline
P(X_1=1\mid Z_1=z_1)
&0.2908947929&0.1307888401
\end{array},
\]
\[
\begin{array}{c|cc}
&Z_2=0&Z_2=1\\
\hline
C=0
&0.4491857282&0.7561593030\\
C=1
&0.4489857282&0.7559593030\\
C=2
&0.4490531731&0.7560267479
\end{array}
\]
where each entry is
$P(X_3=1\mid C=c,Z_2=z_2)$, and
\[
\begin{array}{c|cc}
x_3&0&1\\
\hline
P(X_2=1\mid X_3=x_3)
&0.7817061637&0.7171767286
\end{array}.
\]

For all three tasks, interventions are implemented using the truncated
factorization: the mechanism associated with an intervened variable is
removed, while the remaining mechanisms are unchanged. Consequently,
observations of non-intervened variables can update shared CPT rows and
thereby provide information about the expected rewards of interventions
other than the one selected.

\subsection{Asymptotic Rates of the Monte Carlo Confidence Radii}
\label{app:mc_rates}

We make explicit the dependence of the confidence radii in
Lemmas~\ref{lemm_2} and~\ref{lemm_gt_conc} on the number of posterior
samples $N$. Let $K=|\mathcal{A}|$. Throughout this subsection, we fix
the horizon $T$, the action-set size $K$, and the overall confidence
level $\delta'$. Theorem~\ref{thm2} uses
$\delta=\delta'/(2TK)$ in the per-round, per-action confidence bounds.
Hence, $\delta$ is fixed with respect to $N$, and all logarithmic terms
depending only on $T$, $K$, and $\delta'$ can be treated as constants
when considering the asymptotic dependence on $N$.

First, Lemma~\ref{lemm_2} defines the confidence radius for the expected
regret estimator as
\[
r_N^\Delta
=
\sqrt{\frac{2}{N}
\log\!\left(\frac{2}{\delta}\right)}.
\]
Since $\delta$ is fixed,
$\log(2/\delta)$ is independent of $N$, and therefore
\[
r_N^\Delta
=
\sqrt{2\log(2/\delta)}\,N^{-1/2}
=
O(N^{-1/2}).
\]

We next consider the information-gain radius. Recall from
Lemma~\ref{lemm_gt_conc} that
\[
\beta_N(\delta)
=
\sqrt{\frac{1}{2N}
\log\!\left(\frac{2(3K+2)}{\delta}\right)}.
\]
For fixed $K$ and $\delta$, the logarithmic factor is constant, so
\[
\beta_N(\delta)=O(N^{-1/2}).
\]
The corresponding information-gain confidence radius is
\[
r_N^g
=
\phi_K(K\beta_N)
+
\phi_2(2\beta_N)
+
\phi_{2K}(2K\beta_N),
\]
where, for an alphabet of size $m\geq2$,
\[
\phi_m(\epsilon)
=
h_2\!\left(
\min\left\{\frac{\epsilon}{2},1-\frac{1}{m}\right\}
\right)
+
\min\left\{\frac{\epsilon}{2},1-\frac{1}{m}\right\}
\log(m-1),
\]
and $h_2(u)=-u\log u-(1-u)\log(1-u)$ denotes the binary
entropy function.

We now characterize $\phi_m(\epsilon)$ for small $\epsilon$. Since $m$
is fixed and $\epsilon\to0$, for all sufficiently small $\epsilon$ we
have $\epsilon/2\leq1-1/m$. Hence, with $u=\epsilon/2$,
\[
\phi_m(\epsilon)
=
h_2(u)+u\log(m-1).
\]
Using
\[
h_2(u)
=
u\log\frac{1}{u}
-
(1-u)\log(1-u)
\]
and the elementary inequality
$-\log(1-u)\leq u/(1-u)$ for $0<u<1$, we obtain
\[
-(1-u)\log(1-u)\leq u.
\]
Therefore,
\[
h_2(u)
\leq
u\log\frac{1}{u}+u,
\]
and consequently
\[
\phi_m(\epsilon)
\leq
u\log\frac{1}{u}
+
u\bigl(1+\log(m-1)\bigr).
\]
Since $m$ is fixed, the factor $1+\log(m-1)$ is constant.
Moreover, as $u\to0$, $\log(1/u)\to\infty$, and hence
\[
u
=
O\left(
u\log\frac{1}{u}
\right).
\]
It follows that
\[
\phi_m(\epsilon)
=
O\left(
u\log\frac{1}{u}
\right)
=
O\left(
\epsilon\log\frac{1}{\epsilon}
\right).
\]

We now apply this bound to each term in $r_N^g$. Since $K$ is fixed,
the quantities $K\beta_N$, $2\beta_N$, and $2K\beta_N$ are fixed
multiples of $\beta_N$ and all converge to zero as $N\to\infty$.
Therefore, for all sufficiently large $N$, each of these arguments lies
in the small-$\epsilon$ regime above, and
\(
\phi_K(K\beta_N)
=
O\left(
\beta_N\log\frac{1}{\beta_N}
\right),
\)
\(
\phi_2(2\beta_N)
=
O\left(
\beta_N\log\frac{1}{\beta_N}
\right),
\)
and
\(
\phi_{2K}(2K\beta_N)
=
O\left(
\beta_N\log\frac{1}{\beta_N}
\right).
\)
Summing the three terms gives
\[
r_N^g
=
O\left(
\beta_N\log\frac{1}{\beta_N}
\right).
\]
Because $\beta_N=C/\sqrt{N}$ for a constant $C>0$ depending only on
$K$ and $\delta$,
\[
\log\frac{1}{\beta_N}
=
\frac{1}{2}\log N-\log C
=
O(\log N).
\]
Combining this with $\beta_N=O(N^{-1/2})$ yields
\(
r_N^g
=
O\left(
\frac{\log N}{\sqrt{N}}
\right).
\)

Therefore, for fixed $T$, $|\mathcal{A}|$, and $\delta'$,
\(
r_N^\Delta=O(N^{-1/2}),
\;\;
r_N^g=O\left(\frac{\log N}{\sqrt{N}}\right).
\)
In particular, both radii converge to zero as $N\to\infty$. Hence, for
any fixed $\varepsilon>0$ and any fixed positive denominator separation
$\underline{\rho}_t>0$, the sufficient conditions
\(
r_N^\Delta
\leq
\frac{\varepsilon\underline{\rho}_t}{16}
\)
and
\(
r_N^g
\leq
\min\left\{
\frac{\underline{\rho}_t}{4},
\frac{\varepsilon\underline{\rho}_t^2}{16}
\right\}
\)
are satisfied for all sufficiently large $N$. Equivalently, there
exists a finite $N_0$ such that both conditions hold for every
$N\geq N_0$.

\begin{figure*}[h]
\centering
\begin{subfigure}[b]{5.2cm}
    \centering
    \includegraphics[height=3.5cm,width=5.2cm]{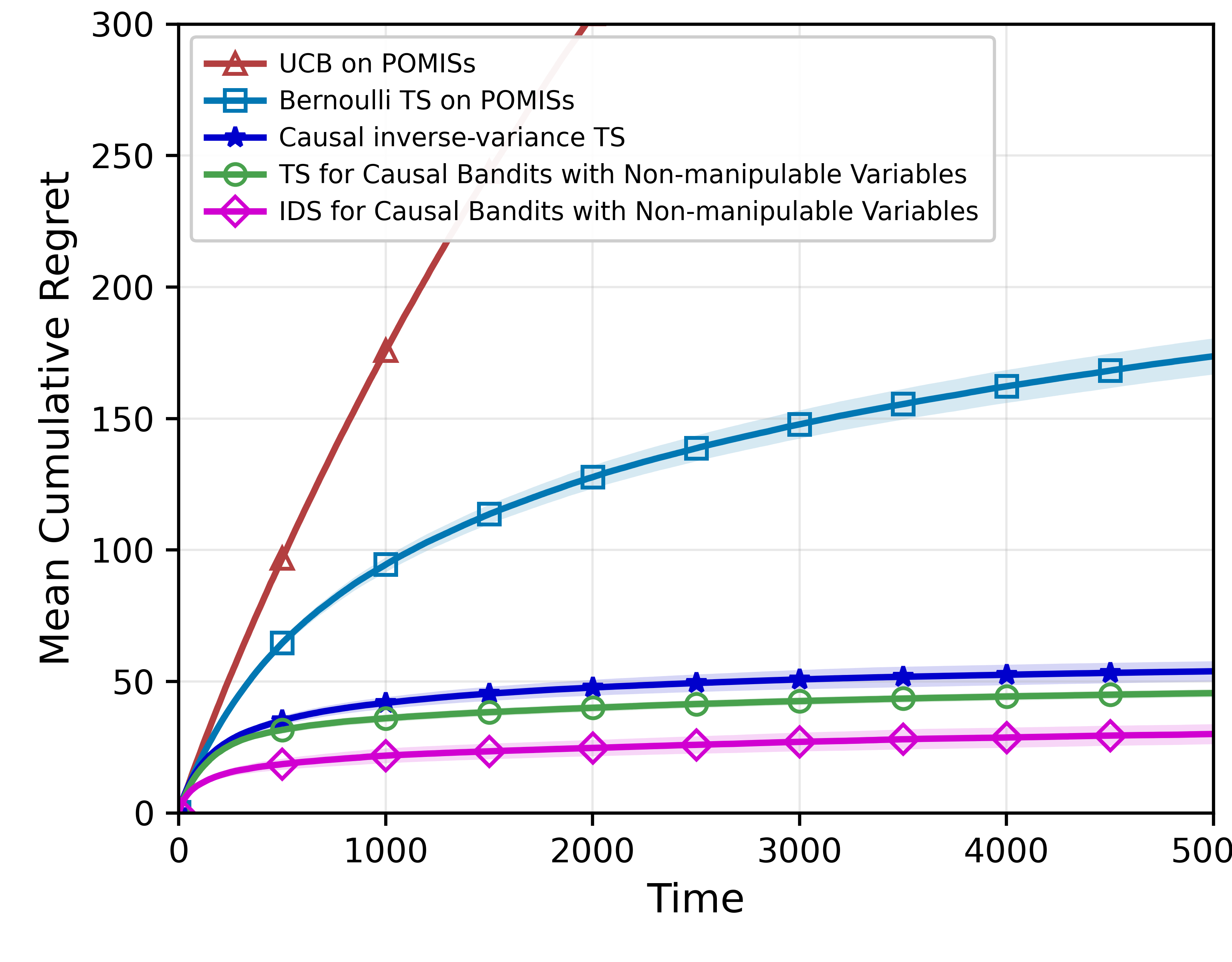}
    \subcaption{$N=250$}
\end{subfigure}
\enspace
\begin{subfigure}[b]{5.2cm}
    \centering
    \includegraphics[height=3.5cm,width=5.2cm]{task3_fig3c.png}
    \subcaption{$N=1000$}
\end{subfigure}
\enspace
\begin{subfigure}[b]{5.2cm}
    \centering
    \includegraphics[height=3.5cm,width=5.2cm]{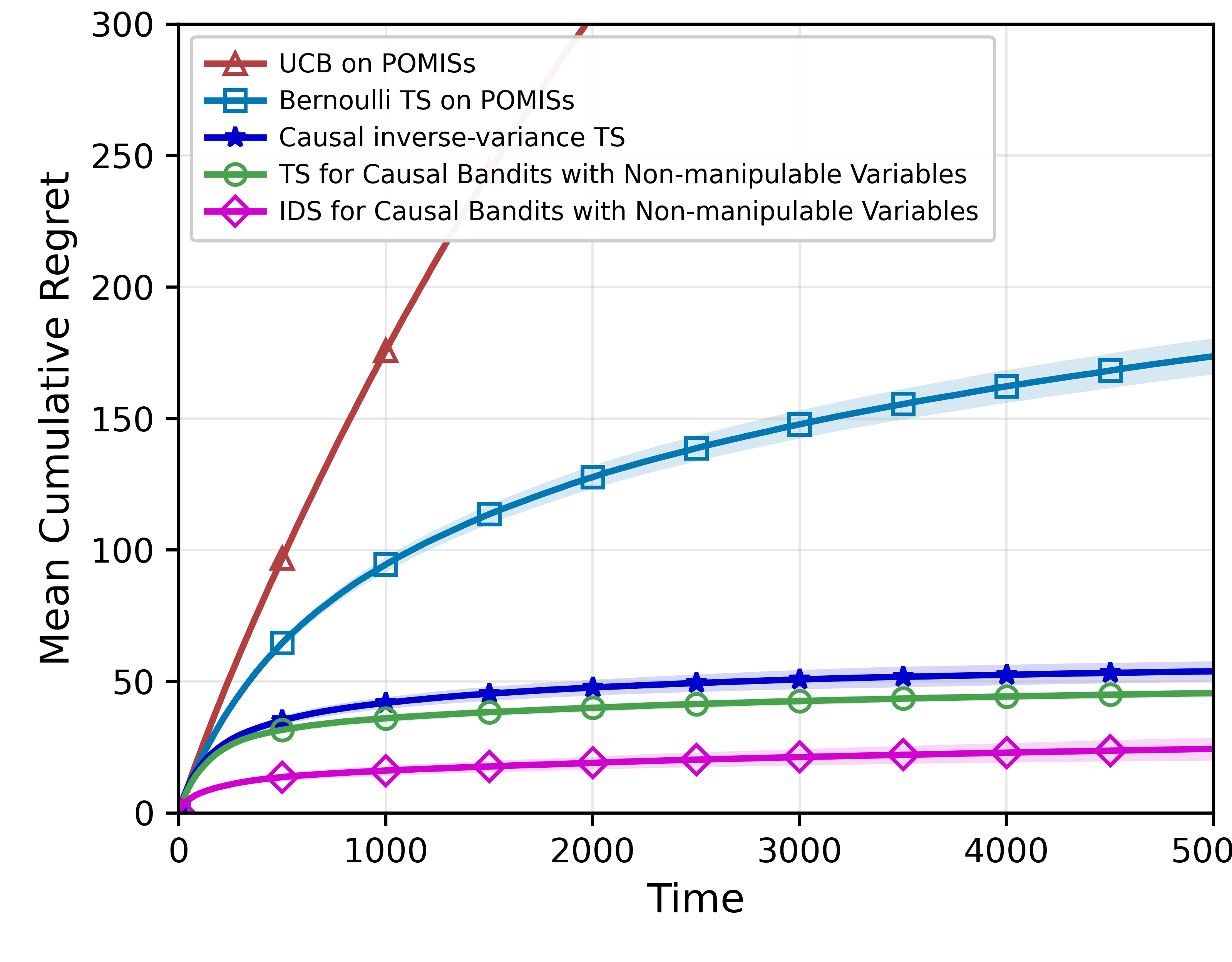}
    \subcaption{$N=4000$}
\end{subfigure}

\caption{Effect of the Monte Carlo budget $N$ on the regret performance
of IDS for Task~3, whose causal graph is shown in
Figure~\ref{exp_graphs}(c). The three panels correspond to $N=250$,
$1000$, and $4000$ posterior samples per round, respectively.}
\label{fig:ids_sample_ablation}
\end{figure*}

\subsection{Sensitivity of IDS to the Monte Carlo Budget}
\label{app_sens}

We evaluate the sensitivity of IDS to the number of posterior samples
$N$ used to approximate the expected regret and information gain. The
ablation is performed on Task~3, which has the largest action set among
the considered tasks, using $N\in\{250,1000,4000\}$. We use $N=1000$
throughout the main experiments.

\end{document}